\documentclass{article}
\usepackage{clipboard}
\usepackage[normalem]{ulem}
\usepackage{xcolor}
\usepackage[main, final]{neurips_2026}

\usepackage[utf8]{inputenc} % allow utf-8 input
\usepackage[T1]{fontenc}    % use 8-bit T1 fonts
\usepackage{hyperref}       % hyperlinks
\usepackage{url}            % simple URL typesetting
\usepackage{booktabs}       % professional-quality tables
\usepackage{amsfonts}       % blackboard math symbols
\usepackage{nicefrac}       % compact symbols for 1/2, etc.
\usepackage{microtype}      % microtypography
\usepackage{xcolor}         % colors
\usepackage{amsmath,amssymb, amsthm} 
\allowdisplaybreaks
\usepackage{xspace}
\usepackage{comment}
\usepackage{makecell}
\usepackage{multirow}
\usepackage{graphicx}
\usepackage[normalem]{ulem}
\usepackage{placeins}
\usepackage{cleveref}

\newtheorem{theorem}{Theorem}[section]

\newcommand{\VE}{\mathcal{V}{\to}\mathcal{E}}
\newcommand{\EV}{\mathcal{E}{\to}\mathcal{V}}

\title{
PANDA: Prior-guided Attentional Dual-path Architecture %
}

\author{%
Rishabh Jain \\
University of Cambridge, UK \\
{\small \texttt{rj412@cst.cam.ac.uk}} \\
\And
Federico Belotti \\
University of Bergamo, Italy \\
{\small \texttt{federico.belotti@unibg.it}} \\
\AND
Stefano Coniglio \\
University of Bergamo, Italy \\
{\small \texttt{stefano.coniglio@unibg.it}} \\
\And
Pietro Li\`{o} \\
University of Cambridge, UK \\
{\small \texttt{pl219@cst.cam.ac.uk}} \\
\And
Stefano Fiorini \\
Independent Researcher \\
{\small \texttt{s.fiorini1994@gmail.com}}
}

\begin{document}

\maketitle

\begin{abstract}
Hypergraph Neural Networks (HNNs), designed to model higher-order relations beyond pairwise interactions, typically rely either on fixed structural weights or on normalised attention. Fixed weights can propagate misleading information in heterophilic settings, while normalisation makes attention coefficients competitive, in the sense that, in a message-passing framework, increasing the weight of one sender must reduce the weight of others in the same aggregation set.
Our main contribution is the introduction of PANDA (Prior-guided AttentioNal Dual-path Architecture), a plug-in for two-stage message passing HNNs
applicable to both undirected and directed settings.
The \emph{main path} interpolates between learned receiver-normalised attention and structural incidence priors, while a non-competitive (not normalised) \emph{auxiliary path} lets each sender contribute its own transformed information.
We also show that, with fixed structural coefficients, PANDA recovers a certain complex-valued convolution operator adopted in previous works.
Across four HNNs spanning spatial and spectral methods in both directed and undirected settings, plugging PANDA in yields an average improvement of $5.18$ percentage points across eight real-world benchmark datasets.
The code is available at: \url{https://github.com/jainris/panda}.
\end{abstract}

\section{Introduction}

Learning from structured, non-Euclidean data has been largely driven by Graph Neural Networks (GNNs), which propagate and aggregate features along pairwise edges~\citep{GNN, noneuclidean}. Many real-world systems, however, involve higher-order relations that cannot be faithfully represented by pairwise links alone, including social networks~\citep{Benson_2016, benson-simplicial}, biological systems~\citep{traversa2023robustnesscomplexitydirectedweighted}, protein
interactions~\citep{murgas2022hypergraph}, and multi-robot systems~\citep{Jain-HMAGAT}. Hypergraph Neural Networks (HNNs) address this limitation by modelling hyperedges as multi-node relations, enabling message passing over sets of interacting entities~\citep{murgas2022hypergraph,chen2022preventing}.

A large class of HNNs can be understood as a two-stage
node$\rightarrow$hyperedge$\rightarrow$node message passing in which features are first aggregated within each hyperedge and then redistributed to its member nodes~\citep{telyatnikov2023hypergraph}.
Despite this common structure, many incidence-aggregation HNNs retain a smoothing bias inherited from fixed structural propagation.
Fixed coefficients, such as degree normalisations or incidence weights, provide stable architecture-specific priors, but they can also propagate misleading information when a hyperedge connects nodes with different labels, roles, or feature distributions. This issue is especially relevant in heterophilic hypergraphs, where the same higher-order relation may contain endpoints that should not share information in the same feature space~\citep{zhu2020beyond,chen2022preventing}.

Attention mechanisms provide one way to adapt aggregation to the data. GAT~\citep{velivckovic2017graph} showed that learned neighbour weights can improve node classification, with GATv2~\citep{brody2022how} later addressing expressivity limitations of the scoring function. Hypergraph attention methods such as HCHA~\citep{bai2021hypergraph} and HyperGAT~\citep{ding2020more} extend this idea to node--hyperedge interactions, yet on standard benchmarks their gains over non-attentive baselines remain modest and often inconsistent across datasets, suggesting that adapting attention to the hypergraph setting is more subtle than a direct port from the pairwise case. A central reason is that these methods rely on receiver-wise normalisation, which forces attention coefficients to be \emph{competitive}: increasing the weight on one sender necessarily reduces the mass available to other senders in the same aggregation set~\citep{richter2020normalized}. This is reasonable when neighbours genuinely compete to explain a target, but it precludes assigning high relevance to multiple senders simultaneously, an outcome that is often desirable in higher-order relations~\citep{yan2025sigmoid}.

Directed hypergraphs add another layer of structure. In this setting, each hyperedge has head and tail roles, and aggregation should distinguish how information flows through these roles. Yet, most directed hypergraph operators still rely on
fixed structural coefficients to weigh head and tail contributions~\citep{DHGNN, fiorini2024let}. The result is a propagation rule whose directional pattern is locked at architecture-design time, independent of the features being propagated. A further open question concerns the spectral operator of~\citet{fiorini2024let}: Does it admit an equivalent two-stage message-passing realisation?

To address these issues, we introduce PANDA (Prior-guided AttentioNal Dual-path Architecture), a plug-in for two-stage hypergraph message-passing layers.
Our approach applies to architectures whose propagation exposes node-to-hyperedge and hyperedge-to-node message passing, or to incidence aggregations that can accept node--hyperedge weights.
At each aggregation stage, PANDA combines a main path, consisting of a structural prior interpolated with learned receiver-normalised scores, with a non-competitive, gated auxiliary path that lets each incidence contribute transformed information without requiring the auxiliary scores to sum to one across the receiver neighbourhood.
In the directed setting, PANDA further factorises its attention along head/tail roles, so that the same plug-in applies to both undirected and directed HNNs.
Additionally, focusing on GeDi-HNN~\citep{fiorini2024let}, we show that it admits a two-stage message--passing realisation over complex-valued representations and that this realisation is not equivalent to transforming the directed hypergraph to a graph via a directed star expansion.
{We evaluate the integration of PANDA into four hypergraph architectures spanning spatial and spectral methods in both directed and undirected settings across eight real-world datasets. Across all models, PANDA consistently improves performance, yielding an average gain of $5.18$ percentage points and improvements of up to $31.01$ points on Telegram~\citep{sen2008collective}.}

\paragraph{Our contributions can be summarised as follows:}
\begin{itemize}
  \item We introduce PANDA, a plug-in for two-stage message-passing architectures, combining a main path that interpolates per-incidence between receiver-normalised attention and the backbone's structural prior with a non-competitive gated auxiliary path. We frame it on directed hypergraphs as they subsume undirected hypergraphs and ordinary graphs as special cases; the mechanism is not hypergraph-specific (we demonstrate its effectiveness on a GCN backbone as well).
  \item We prove that GeDi-HNN's complex Hermitian convolution admits a two-stage node--hyperedge--node realisation through the complex incidence matrix and its conjugate (\Cref{thm:gedi-recovery}), and that this realisation is not recovered by applying the same operator after a directed star expansion, which differs in phase (\Cref{thm:gedi-star-expansion}).
  \item Across 4 backbones spanning spectral/spatial and directed/undirected architectures on eight benchmarks, PANDA yields an average gain of +5.18 pp, with gains of up to +31.01 pp.
  \item We show that non-competitiveness is what drives the gains: against a normalised auxiliary path with identical parameter count and identical separate transformation, PANDA wins on 27 of 32 model-dataset pairs (+2.60 pp average).
  \item We provide a mechanistic analysis of the learned routing, separating the auxiliary path's contribution (train-time ablation, \Cref{tab:eval-merged}) from the trained model's dependence on it (inference-time interventions), and show its predictive effect is a broad margin improvement rather than only neighbour-majority correction.
\end{itemize}

\section{Related Work}

\paragraph{Undirected and Directed Hypergraph Neural Networks.}

On the \emph{undirected} side, HNNs split into spectral and spatial families. Spectral methods build convolution operators from the eigendecomposition of a hypergraph Laplacian, drawing on graph signal processing and the graph Fourier transform; examples include
HGNN~\citep{feng2019hypergraph}, HCHA~\citep{bai2021hypergraph} and HGNN$^{+}$~\citep{gao2022hgnn}, all of which can be viewed as GNNs operating on the clique expansion of a hypergraph. Spatial methods instead frame convolution as a localised aggregation over hyperedge neighbourhoods.
HNHN~\citep{dong2020hnhn}, UniGCNII~\citep{UniGNN}, HyperDN~\citep{tudisco2021nonlinea}, and LEGCN~\citep{yang2022hypergraph}, together with set-based models such as AllDeepSets and AllSetTransformer~\citep{allset}, propagate hyperedge features within a message-passing framework and can also be interpreted as GNNs on the star expansion. PhenomNN~\citep{wang2023hypergraph} derives its architecture from hypergraph-regularised energy functions, combining both the clique and star expansions.

On the \emph{directed} side, work is comparatively sparse.
\citet{fiorini2023sigmanetlaplacianrule} propose SigMaNet for directed graphs.
DHRL~\citep{DHGNN} approximates a directed hypergraph Laplacian and defines convolutions on top of it, while GeDi-HNN~\citep{fiorini2024let} introduces a graph-convolution operator based on the Generalized Directed Laplacian, a complex-valued Hermitian matrix shown to subsume many existing directed and undirected (hyper)graph convolutions as special cases.

\paragraph{Attention in GNNs and HNNs.} GAT~\citep{velivckovic2017graph} first showed that learned similarity-based weighting of neighbour contributions improves node classification, and GATv2~\citep{brody2022how} later introduced a strictly more expressive variant.
In the HNN domain, early proposals such as HCHA~\citep{bai2021hypergraph} and HyperGAT~\citep{ding2020more} incorporate hyperedge-level attention but remain undirected and adopt a purely competitive, receiver-wise normalisation: increasing the weight on one sender necessarily reduces the mass available to others in the same receiver neighbourhood.

\paragraph{Gated Attention.} Gated attention mechanisms augment standard attention with an additional learned gate that modulates the information passed by an attention head, feature channel, message, or residual path. While attention weights select where to aggregate from, gates control how much of the resulting signal should be admitted. This idea follows the broader use of gates for adaptive information flow in deep networks \citep{srivastava2015highway,dauphin2017language,shazeer2020glu}, and has been instantiated in attention models at different granularities: over attention heads, over value/output features, or over residual paths \citep{parisotto2020stabilizing,hua2022transformer,yang2024gated}. In graph learning, GaAN introduced gates over multi-head graph attention, using a lightweight subnetwork to control the importance of each attention head \citep{zhang2018gaan}. More recently, \cite{qiu2025gatedattention} showed that a sigmoid gate applied after attention improves performance, training stability, and scaling behaviour, while inducing sparse modulation of the attention output and enhancing long-context extrapolation. Notably, all these existing methods apply gating in conjunction with competitive normalisation (usually after it), unlike our proposed work. Furthermore, these works do not explore applications to hypergraphs.

\section{PANDA: Prior-guided AttentioNal Dual-path Architecture}\label{sec:gadh}

Let $\mathcal{H} = (\mathcal{V}, \mathcal{E})$ be a directed hypergraph with node features $\mathbf{x}^{(\ell)}_u \in \mathbb{R}^d$ for node $u \in \mathcal{V}$ and hyperedge features $\mathbf{h}_e \in \mathbb{R}^d$, $e \in \mathcal{E}$.
The hyperedge weights are stored in the diagonal matrix $W \in \mathbb{R}^{m \times m}$.
The vertex and hyperedge degrees are defined as $d_u = \sum_{e \in \mathcal{E}: u \in e} |w_e|$ for $u \in \mathcal{V}$ and $\delta_e = |e|$ for $e \in \mathcal{E}$ and are stored in two diagonal matrices $D_v \in \mathbb{R}^{n \times n}$ and $D_e \in \mathbb{R}^{m \times m}.$
Following~\cite{gallo1993directed}, the direction of a hyperpedge $e \in \mathcal{E}$ is determined by the partition in tail set $T(e) \subseteq \mathcal{V}$ and head set $H(e) \subseteq \mathcal{V}$.
For each $u \in \mathcal{V}$, $\Gamma_H(u)$ and $\Gamma_T(u)$ denote the sets of hyperedges in which $u$ appears as a head or tail node, respectively, while $\Gamma(u)$ denotes the set of all hyperedges incident to $u$, i.e. $\Gamma(u) = \Gamma_H(u) \cup \Gamma_T(u)$. For undirected hypergraphs, we use the same notation by placing every incident node in the head set: $H(e)=e$ and $T(e)=\emptyset$, so $\Gamma_H(u)$ contains all hyperedges incident to $u$ and $\Gamma_T(u)=\emptyset$.

The following framework describes a general {\em two-stage} node--hyperedge--node message-passing
abstraction~\citep{allset,telyatnikov2023hypergraph}:
{\small
\begin{align}
  \mathbf{h}^{(\ell)}_e & = \mathrm{Agg}_{\mathcal{E}}^{(\ell)}\!\left(
  \left\{
    \mathrm{Trans}_{\mathcal{V}}^{(\ell)}\!\left(
    \mathbf{x}^{(\ell-1)}_v
    \right)
  \right\}_{v \in e}
  \right)                     & e \in \mathcal{E}
  \label{eq:agg-e}
  \\
  \mathbf{x}^{\prime(\ell)}_u & = \mathrm{Agg}_{\mathcal{V}}^{(\ell)}\!\left(
  \left\{
      \mathrm{Trans}_{\mathcal{E}}^{(\ell)}\!\left(
          \mathbf{h}^{(\ell)}_e
      \right)
  \right\}_{e \in \Gamma(u)}
  \right)                     & u \in \mathcal{V}
  \label{eq:agg-v}
  \\
  \mathbf{x}^{(\ell)}_u       & = \mathrm{Upd}_{\mathcal{V}}^{(\ell)}\!\left(
  \mathbf{x}^{(\ell-1)}_u,\, \mathbf{x}^{\prime(\ell)}_u
  \right)                     & u \in \mathcal{V}.
  \label{eq:upd-v}
\end{align}
}\noindent
Equation \ref{eq:agg-e} carries out the first stage of message passing, aggregating transformed embeddings of the relevant nodes, while Equations \ref{eq:agg-v} and \ref{eq:upd-v} carry out the second stage, updating the node embeddings using the aggregated hyperedge representations.
$\mathrm{Trans}^{(\ell)}_{\cdot}$, $\mathrm{Agg}^{(\ell)}_{\cdot}$, and $\mathrm{Upd}^{(\ell)}_{\cdot}$ are learnable functions; their specific instantiation defines different model architectures. For instance, $\mathrm{Agg}_{\mathcal{E}}$ and $\mathrm{Agg}_{\mathcal{V}}$ may be weighted sums over incident nodes and hyperedges, as in GeDi-HNN~\cite{fiorini2024let};
learned permutation-invariant multiset functions, as in AllSet and AllSetTransformer~\citep{allset}; or attention-weighted sums, as in hypergraph attention models such as HCHA and HyperGAT~\citep{bai2021hypergraph,ding2020more}.

PANDA is an attention plug-in that can be applied to virtually any two-stage message passing architecture  by intercepting the node-to-hyperedge and hyperedge-to-node incidence messages used in Equations~\eqref{eq:agg-e} and \eqref{eq:agg-v}, and manipulating them.
It consists of two paths: the main path interpolates between learned receiver-normalised attention and structural incidence priors, while an auxiliary path adds locally complement-bounded but non-normalised transformed messages. We describe them in detail in the following.

\subsection{PANDA's Main Path}

After intercepting the node-to-hyperedge and hyperedge-to-node incidence messages used in Equations~\eqref{eq:agg-e} and \eqref{eq:agg-v}, PANDA's main path re-weights them by the main-path attention coefficients $\textcolor{cyan}{\hat\alpha^{\mathrm{main},(\ell)}_{e,v}}$ and $\textcolor{orange}{\check\alpha^{\mathrm{main},(\ell)}_{u,e}}$. The updated equations for the target layer reweighted by the new attention scores read:
{\small
\begin{align}
  \label{eq:upd_undirected_att}
  \mathbf{h}_e^{(\ell)}
          & = \mathrm{Agg}_{\mathcal{E}}^{(\ell)}\!\left(
  \left\{
  \textcolor{cyan}{\hat{\alpha}^{\mathrm{main},(\ell)}_{e,v}}
  \mathbf{x}_v^{\mathrm{main},(\ell)}
  \right\}_{v \in e}
  \right) &
  \quad e \in \mathcal{E}                                 \\
  \label{eq:upd_undirected_att2}
  \mathbf{x}_u^{\prime(\ell)}
          & = \mathrm{Agg}_{\mathcal{V}}^{(\ell)}\!\left(
  \left\{\textcolor{orange}{\check{\alpha}^{\mathrm{main},(\ell)}_{u,e}} \mathbf{h}_e^{\mathrm{main},(\ell)}\right\}_{e \in \Gamma(u)}
  \right) &
  \quad u \in \mathcal{V},
\end{align}
}\noindent
where we define $\mathbf{x}_v^{\mathrm{main},(\ell)}$ and $\mathbf{h}_e^{\mathrm{main},(\ell)}$ as below for better readability:
{\small
\begin{align}
    \mathbf{x}_v^{\mathrm{main},(\ell)} & = \mathrm{Trans}_{\mathcal{V}}^{(\ell)}\!\left(
          \mathbf{x}^{(\ell-1)}_v
      \right) & v \in \mathcal{V} \\
    \mathbf{h}_e^{\mathrm{main},(\ell)} & = \mathrm{Trans}_{\mathcal{E}}^{(\ell)}\!\left(
          \mathbf{h}^{(\ell)}_e
      \right) & e \in \mathcal{E}.
\end{align}
}

PANDA's main-path attention mechanism extends GATv2's~\citep{brody2022how} by combining GATv2's coefficients $\hat\alpha^{\mathrm{gat},(\ell)}_{e,v}$ and $\check\alpha^{\mathrm{gat},(\ell)}_{u,e}$ with the structural prior coefficients $\hat\alpha^{\mathrm{prior}}_{e,v}$ and $\check\alpha^{\mathrm{prior}}_{u,e}$. These are the fixed coefficients used by the target architecture in the node-to-hyperedge message from $v$ to $e$ and the hyperedge-to-node message from $e$ to $u$. These coefficients may come from degree normalisation, hyperedge weights, incidence signs or phases, or model-specific hyperedge reduction scores.
Setting the $\beta$ coefficients to the prior-only regime, $\hat\beta^{(\ell)}_{e,v}=\check\beta^{(\ell)}_{u,e}=0$, recovers these fixed coefficients in the main routing path.

Following~\cite{brody2022how}, the scores $\hat\alpha^{\mathrm{gat},(\ell)}_{e,v}$ and $\check\alpha^{\mathrm{gat},(\ell)}_{u,e}$ are defined as follows:
{\small
\begin{align}
  \hat\alpha^{\mathrm{gat},(\ell)}_{e,v}
   & =
  \frac{
    \exp\!\left(\hat s_{e,v}
    \right)
  }{
    \displaystyle
    \sum_{v' \in H(e) \cup T(e)}
    \exp\!\left(\hat s_{e,v'}
    \right)
  } \quad \hat s_{e,v} = (\hat{\mathbf{a}}^{(\ell)})^{\top} \operatorname{LeakyReLU}
  \left( \hat{\mathbf{W}}^{(\ell)}
  \left[
    \mathbf{h}^{(\ell-1)}_{e}
    \,\Vert\,
    \mathbf{x}^{(\ell-1)}_{v}
    \right]
  \right)
  \label{eq:alpha_att_ev}
  \\[2pt]
  \check\alpha^{\mathrm{gat},(\ell)}_{u,e}
   & =
  \frac{
    \exp\!\left(\check s_{u,e}
    \right)
  }{
    \displaystyle
    \sum_{e' \in \Gamma_H(u) \cup \Gamma_T(u)}
    \exp\!\left(\check s_{u, e'}
    \right)
  } \quad \check s_{u,e} = (\check{\mathbf{a}}^{(\ell)})^{\top} \operatorname{LeakyReLU}
  \left( \check{\mathbf{W}}^{(\ell)}
  \left[
    \mathbf{x}^{(\ell-1)}_{u}
    \,\Vert\,
    \mathbf{h}^{(\ell)}_{e}
    \right]
  \right),
  \label{eq:alpha_att_ue}
\end{align}
}\noindent
with learnable parameters $\hat{\mathbf{a}}^{(\ell)},\check{\mathbf{a}}^{(\ell)} \in \mathbb{R}^{d}$,
$\hat{\mathbf{W}}^{(\ell)},\check{\mathbf{W}}^{(\ell)} \in \mathbb{R}^{d \times 2d}$ and $\Vert$ is the concatenation operation.

In PANDA's main path, the coefficients $\hat\alpha^{\mathrm{gat},(\ell)}_{e,v}, \check\alpha^{\mathrm{gat},(\ell)}_{u,e}$ and $\hat\alpha^{\mathrm{prior}}_{e,v}, \check\alpha^{\mathrm{prior}}_{u,e}$ are combined via the following per-representation convex combination with $\hat{\beta}^{(\ell)}_{e,v}, \check{\beta}^{(\ell)}_{u,e} \in [0, 1]$:
\begin{align}
  \textcolor{cyan}{\hat\alpha^{\mathrm{main},(\ell)}_{e,v}}     & = \hat{\beta}^{(\ell)}_{e,v}\, \hat\alpha^{\mathrm{gat},(\ell)}_{e,v} + \left(1 - \hat{\beta}^{(\ell)}_{e,v}\right) \hat\alpha^{\mathrm{prior}}_{e,v},         \\
  \textcolor{orange}{\check\alpha^{\mathrm{main},(\ell)}_{u,e}} & = \check{\beta}^{(\ell)}_{u,e}\, \check\alpha^{\mathrm{gat},(\ell)}_{u,e} + \left(1 - \check{\beta}^{(\ell)}_{u,e}\right) \check\alpha^{\mathrm{prior}}_{u,e}.
\label{eq:main-coefficients}
\end{align}
Notice that, since the interpolation gates $\hat{\beta}^{(\ell)}_{e,v}, \check{\beta}^{(\ell)}_{u,e}$ are used incidence-wise, the resulting $\textcolor{cyan}{\hat\alpha^{\mathrm{main},(\ell)}_{e,v}}, \textcolor{orange}{\check\alpha^{\mathrm{main},(\ell)}_{u,e}}$ coefficients are not necessarily receiver-wise normalised, even when the GAT coefficients and the chosen structural prior are each separately normalised over the receiver neighbourhood.
The learned gating coefficients $\hat{\beta}^{(\ell)}_{e,v}, \check{\beta}^{(\ell)}_{u,e} \in [0, 1]$ are calculated as follows:
{\small
\begin{align}
  \hat{\beta}_{e,v}^{(\ell)}
   & =
  \sigma_s
  \left( \left(
  \mathbf{u}_{\beta,\mathcal V}^{(\ell)} \right)^\top
  \operatorname{LeakyReLU}
  \left(
    {W}_{\beta,\mathcal V}^{(\ell)}
    \left[
      \mathbf{h}^{(\ell-1)}_e
      \,\Vert\,
      \mathbf{x}^{(\ell-1)}_v
      \right]
    \right)
  \right)
  \\
  \check{\beta}_{u,e}^{(\ell)}
   & =
  \sigma_s
  \left( \left(
  \mathbf{u}_{\beta,\mathcal E}^{(\ell)} \right)^\top
  \operatorname{LeakyReLU}
  \left(
    {W}_{\beta,\mathcal E}^{(\ell)}
    \left[
      \mathbf{x}^{(\ell-1)}_u
      \,\Vert\,
      \mathbf{h}^{(\ell)}_e
      \right]
    \right)
  \right),
\end{align}
}\noindent
with learnable $\mathbf{W}^{(\ell)}_{\beta, \cdot} \in \mathbb{R}^{c \times 2d}$ and $\mathbf{u}^{(\ell)}_{\beta, \cdot} \in \mathbb{R}^{c}$, and where $\sigma_s(\cdot)$ is the sigmoid function.
Setting $\hat{\beta}^{(\ell)}_{e,v}=\check{\beta}^{(\ell)}_{u,e}=1$ recovers the GATv2 attentional coefficients, while setting $\hat{\beta}^{(\ell)}_{e,v}=\check{\beta}^{(\ell)}_{u,e}=0$ recovers the structural prior coefficients (in the main routing path).

Thanks to the $\beta$ coefficients, the attention mechanism can shift towards the structural prior on representations where it is reliable (e.g., homophilic neighbourhoods in which smoothing is desirable and, thus, the prior is beneficial), while retaining the learned GATv2 scores where the prior is misleading (e.g., heterophilic settings).

\subsection{Auxiliary Path}
\label{sec:complement-gated-attention}

PANDA adds an  {\em auxiliary aggregation mechanism} to the one already featured in the target architecture (after reweighting it with the coefficients $\hat\alpha^{(\ell)}_{e,v}$ and $\check\alpha^{(\ell)}_{u,e}$).
We introduce this auxiliary aggregation mechanism by carrying out the following transformations
  {\small
    \begin{align}
      \mathbf{x}^{\mathrm{aux},(\ell)}_v & =
      \mathrm{Trans}_{\mathcal{V}}^{\mathrm{aux},(\ell)}\!\left(
      \mathbf{x}^{(\ell-1)}_v
      \right)                            & v \in \mathcal{V}
      \label{eq:trans-v-par}
      \\
      \mathbf{h}^{\mathrm{aux},(\ell)}_e & =
      \mathrm{Trans}_{\mathcal{E}}^{\mathrm{aux},(\ell)}\!\left(
      \mathbf{h}^{(\ell-1)}_e
      \right)                            & e \in \mathcal{E}.
      \label{eq:trans-e-par}
    \end{align}
  }\noindent
We update the aggregation operations of Equations~\eqref{eq:upd_undirected_att}--\eqref{eq:upd_undirected_att2} by additively combining the transformed representations $\mathbf{h}_e^{\mathrm{aux},(\ell)}$ and $\mathbf{x}_v^{\mathrm{aux},(\ell)}$ to the original, reweighted ones:
{\small
\begin{align}
  \label{eq:upd_comp_att}
  \mathbf{h}^{(\ell)}_e                                                                                                                                                                                                                  & = \mathrm{Agg}_{\mathcal{E}}^{(\ell)}\!\left(
  \left\{\textcolor{cyan}{\hat{\alpha}^{\mathrm{main},(\ell)}_{e,v}} \mathbf{x}^{\mathrm{main},(\ell)}_v\right\}_{v \in e} \oplus \left\{\textcolor{red}{\hat{\alpha}^{\text{aux},(\ell)}_{e,v}} \mathbf{x}_v^{\mathrm{aux},(\ell)}\right\}_{v \in e}\right) & e \in \mathcal{E}                             \\
  \mathbf{x}^{\prime(\ell)}_u                                                                                                                                                                                                                  & = \mathrm{Agg}_{\mathcal{V}}^{(\ell)}\!\left(
  \left\{\textcolor{orange}{\check{\alpha}^{\mathrm{main},(\ell)}_{u,e}} \mathbf{h}^{\mathrm{main},(\ell)}_e\right\}_{e \in \Gamma(u)} \oplus \left\{\textcolor{blue}{\check{\alpha}^{\text{aux},(\ell)}_{u,e}} \mathrm{h}_e^{\mathrm{aux},(\ell)} \right\}_{e \in \Gamma(u)}
  \right)                                                                                                                                                                                                                                      & u \in \mathcal{V},
\end{align}
}\noindent
where $\oplus$ indicates the per-index sum of two indexed sets that share the same indices.

As weights, we use a gated complement of the main path scores:
{\small
\begin{align}
  \textcolor{red}{\hat\alpha^{\mathrm{aux},(\ell)}_{e,v}}    & = \left(1 - \textcolor{cyan}{\hat\alpha^{\mathrm{main},(\ell)}_{e,v}}\right) \textcolor{red}{\hat g^{(\ell)}_{e,v}},
  \qquad
  \textcolor{red}{\hat g^{(\ell)}_{e,v}} := \sigma_s\!\left(\mathrm{aux}_{\mathcal{E}}^{(\ell)}\!\left(\mathbf{h}^{(\ell-1)}_{e},\,\mathbf{x}^{(\ell-1)}_{v} \right)\right) \\
  \textcolor{blue}{\check\alpha^{\mathrm{aux},(\ell)}_{u,e}} & = \left(1 - \textcolor{orange}{\check\alpha^{\mathrm{main},(\ell)}_{u,e}}\right) \textcolor{blue}{\check g^{(\ell)}_{u,e}},
  \qquad
  \textcolor{blue}{\check g^{(\ell)}_{u,e}} := \sigma_s\!\left(\mathrm{aux}_{\mathcal{V}}^{(\ell)}\!\left(\mathbf{x}^{(\ell-1)}_{u},\, \mathbf{h}^{(\ell)}_{e}\right)\right),
\end{align}
}\noindent
where $\sigma_s(\cdot)$ is the sigmoid function, the per-incidence values $\textcolor{red}{\hat g^{(\ell)}_{e,v}}\in[0,1]$ and $\textcolor{blue}{\check g^{(\ell)}_{u,e}}\in[0,1]$ are the post-sigmoid auxiliary gates, and the functions $\mathrm{aux}_{\mathcal{E}}^{(\ell)},\mathrm{aux}_{\mathcal{V}}^{(\ell)}$ are learned scalar scoring functions:
{\small
\begin{align}
  \mathrm{aux}_{\mathcal E}^{(\ell)}
  \!\left(\mathbf{h}^{(\ell-1)}_e,\mathbf{x}^{(\ell-1)}_v\right)
   & = (\mathbf{a}_{\hat\alpha^{\mathrm{aux}}}^{(\ell)})^\top
  \operatorname{ReLU}\!\left(
  \mathbf{W}_{\hat\alpha^{\mathrm{aux}}}^{(\ell)}
  \left[
    \mathbf{h}^{(\ell-1)}_e
    \,\Vert\,
    \mathbf{x}^{(\ell-1)}_v
    \right]
  \right)
  \\
  \mathrm{aux}_{\mathcal V}^{(\ell)} \!\left(\mathbf{x}^{(\ell-1)}_u,\mathbf{h}^{(\ell)}_e\right)
   & = (\mathbf{a}_{\check\alpha^{\mathrm{aux}}}^{(\ell)})^\top \operatorname{ReLU}\!\left(
  \mathbf{W}_{\check\alpha^{\mathrm{aux}}}^{(\ell)}
  \left[
    \mathbf{x}^{(\ell-1)}_u
    \,\Vert\,
    \mathbf{h}^{(\ell)}_e
    \right]
  \right).
\end{align}
}\noindent
with learned $\mathbf{W}_{\hat\alpha^{\mathrm{aux}}}^{(\ell)}, \mathbf{W}_{\check\alpha^{\mathrm{aux}}}^{(\ell)} \in\mathbb{R}^{d\times 2d}$ and $\mathbf{a}_{\hat\alpha^{\mathrm{aux}}}^{(\ell)}, \mathbf{a}_{\check\alpha^{\mathrm{aux}}}^{(\ell)} \in\mathbb{R}^{d}$.

Assume that the structural priors are scaled so that the resulting main coefficients lie in $[0,1]$. Then the auxiliary gates satisfy
\[
  0 \leq \hat\alpha^{\mathrm{aux},(\ell)}_{e,v}
  \leq 1-\hat\alpha^{\mathrm{main},(\ell)}_{e,v},
  \qquad
  0 \leq \check\alpha^{\mathrm{aux},(\ell)}_{u,e}
  \leq 1-\check\alpha^{\mathrm{main},(\ell)}_{u,e}.
\]
Setting the auxiliary coefficients, or the auxiliary transformations, to zero recovers the prior-guided main-path reweighting of the target layer.

Notice that these are sender--receiver-wise bounds: $1-\alpha^{\mathrm{main}}$ is the local complement of the corresponding sender--receiver coefficient.
In particular, the auxiliary coefficients are not normalised across $v\in e$ or across $e\in\Gamma(u)$.
The main path blends the receiver-normalised GATv2 component with the structural prior, while the auxiliary path adds a separately transformed message whose coefficient is locally complement-bounded but non-competitive across the receiver's receptive field. Thus, interactions assigned a small main-path coefficient can still contribute after transformation, which is useful in heterophilic or role-sensitive hyperedges where a sender may be poor as a smoothing neighbour but informative after a separate transformation.

\section{Generalised Directed Spectral Convolution as Two-Stage Message Passing}
\label{sec:method-2}

In this section, we analyse the relationship between two-stage message-passing and the only, to our knowledge, Hermitian convolutional operators defined for directed hypergraphs, namely, the operator used in GeDi-HNN and introduced by~\cite{fiorini2024let}. The analysis is crucial to understand how to plug PANDA in such an architecture.

The relation between spectral graph convolutions and spatial/message-passing graph convolutions is well established for ordinary graphs,
see~\citep{balcilar2020bridging,balcilar2021analyzing,chen2023bridging}.
For hypergraphs, the two-stage node--hyperedge--node message-passing abstraction is also standard in the undirected setting. HGNN and HCHA re-cast the Zhou--Huang--Sch\"{o}lkopf operator as a node--hyperedge--node message-passing scheme, and~\citep{zhou,feng2019hypergraph,bai2021hypergraph} and, more broadly, AllSet formalises HNN layers as compositions of permutation-invariant multiset functions~\citep{allset}, while~\citet{telyatnikov2023hypergraph} place HGNN, HCHA, HNHN, HyperGCN, UniGNN, and AllSet-type models under the same two-stage viewpoint.
What is not immediate from these existing results is the directed, complex-Hermitian case considered by models such as GeDi-HNN~\citep{fiorini2024let}.

In this section, we show that, to recover GeDi's convolutional operator, one needs a complex-valued message-passing scheme relying on the same complex incidence matrix and its conjugate used to define GeDi's Laplacian. We also show that the directed, complex-valued GeDi operator implemented by GeDi's Laplacian admits a node--hyperedge--node two-stage equivalent, but not by (as one may intuit) first converting the hypergraph into an ordinary directed star expansion and then applying a Hermitian graph operator.

\subsection{A complex-valued message-passing scheme}
\label{sec:gedi-complex}

To compare our two-stage template with GeDi's spectral convolution, we instantiate the two aggregation functions $\mathrm{Agg}_{\mathcal{E}}^{(\ell)}$, $\mathrm{Agg}_{\mathcal{V}}^{(\ell)}$ as follows:
{\small
\begin{align}
  \mathbf{h}^{(\ell)}_{e}
          & = \sum_{v \in H(e) \cup T(e)} \vec{B}^*_{e,v}
  \left(
  \textcolor{cyan}{\hat\alpha^{\text{main},(\ell)}_{e,v}}\,
  \mathbf{x}^{\mathrm{main},(\ell)}_{v}
  +
  \textcolor{red}{\hat\alpha^{\mathrm{aux},(\ell)}_{e,v}}\,
  \mathbf{x}^{\mathrm{aux},(\ell)}_{v}
  \right) & e \in \mathcal{E}
  \label{eq:gadh-h}
  \\
  \mathbf{x}^{\prime(\ell)}_{u}
          & = \sum_{e \in \Gamma_H(u) \cup \Gamma_T(u)} \vec{B}_{u,e}
  \left(
  \textcolor{orange}{\check\alpha^{\text{main},(\ell)}_{u,e}}\,
  \mathbf{h}^{\mathrm{main},(\ell)}_{e}
  +
  \textcolor{blue}{\check\alpha^{\mathrm{aux},(\ell+1)}_{u,e}}\,
  \mathbf{h}^{\mathrm{aux},(\ell)}_{e}
  \right) & u \in \mathcal{V}
  \label{eq:gadh-xprime},
\end{align}
}\noindent
where $\vec{B}^* \in \mathbb{C}^{n \times m}$ is the conjugate-transpose of $\vec{B} \in \mathbb{C}^{n \times m}$, which is \cite{fiorini2024let}'s complex vertex--edge incidence matrix
\begin{equation}\label{eq:B}
  \vec{B}_{v,e} =
  \begin{cases}
    1  & \text{if } v \in H(e) \\
    -i & \text{if } v \in T(e) \\
    0  & \text{otherwise}
  \end{cases}
  \qquad v \in \mathcal{V},\, e \in \mathcal{E}.
\end{equation}
The role of $\vec{B}$ is to encode the hyperedge directions and is central to the following theorem:
\begin{theorem}\label{thm:gedi-recovery}
\Copy{cp-thm:gedi-recovery}{%
    PANDA with additive aggregations $\mathrm{Agg}_{\cdot}$, $\mathrm{Trans}^{(\ell)}_{\cdot}=\mathrm{id}$, $\mathrm{Trans}^{\mathrm{aux},(\ell)}_{\cdot}=\mathbf{0}$, $\mathrm{Upd}^{(\ell)}(X,X')=\phi(X\Theta_0+X'\Theta_1)$, with $\hat\beta^{(\ell)}_{e,v}=\check\beta^{(\ell)}_{u,e}=0$ (so that $\hat\alpha^{\mathrm{main},(\ell)}_{e,v}=\hat\alpha^{\mathrm{prior}}_{e,v}$ and $\check\alpha^{\mathrm{main},(\ell)}_{u,e}=\check\alpha^{\mathrm{prior}}_{u,e}$), $\hat\alpha^{\mathrm{aux},(\ell)}_{e,v}=\check\alpha^{\mathrm{aux},(\ell)}_{u,e}=0$, and where the structural priors are
    \[
      \hat\alpha^{\mathrm{prior}}_{e,v} =  \sqrt{\frac{w_e}{\delta_e \, d_v}},
      \qquad
      \check\alpha^{\mathrm{prior}}_{u,e} = \sqrt{\frac{w_e}{d_u \, \delta_e}}.
    \]
    recovers GeDi's convolutional layer defined on top of the Generalized Directed (GeDi) Laplacian $\vec{L}_N = I-\vec{Q}_N$, with $\vec{Q}_N := D_v^{-1/2}\,\vec{B}\,W\,D_e^{-1}\,\vec{B}^*\,D_v^{-1/2}$.
}
\end{theorem}

Notice that, when the coefficients are learned, PANDA does not necessarily preserve the Hermitianicity of the spectral operator which otherwise holds in the frozen-prior case.

In the undirected hypergraph case, a spectral convolution based on the Zhou--Huang--Sch\"{o}lkopf Laplacian can be recovered as a two-stage message-passing operator applied to the star expansion of the hyperpgraph. This is not the case in the directed setting, where the phases encoded in the Hermitian matrix do not line up, as we now show:

\begin{theorem}\label{thm:gedi-star-expansion}
\Copy{cp-thm:gedi-star-expansion}{%
    Let $\mathcal{G}_\star$ be the directed star expansion obtained from a directed hypergraph $\mathcal{H}$ by replacing each hyperedge $e$ with an auxiliary node $h_e$, adding arcs $t\to h_e$ for every $t\in T(e)$ and arcs $h_e\to r$ for every $r\in H(e)$.
    Let $\vec{Q}_\star$ and $\vec{L}_\star$ be the GeDi Signless and Signed Laplacian
    operators on this directed graph, and let $\vec{Q}_N$ be the GeDi normalised Signless Laplacian on $\mathcal{H}$. The block  $(\vec{Q}_\star^2)_{\mathcal V,\mathcal V}$ is the two-hop node--auxiliary--node propagation obtained by applying GeDi twice on the
    star-expanded graph. In general, however, neither the one-step operator
    $\vec{Q}_\star$ nor this two-hop star-expanded propagation equals $\vec{Q}_N$ on the original vertices.
    Thus even the two-hop message passing induced by the ordinary star expansion
    already differs from $\vec Q_N$ in phase.
}
\end{theorem}

\section{Experimental Evaluation}

\subsection{Experimental Setup}
\label{sec:eval-setup}

\paragraph{Baseline.}
We evaluate our PANDA modular plug-in applied to four HNNs on real-world datasets for the node classification task. The selected models span complementary regions of the spectral/spatial and directed/undirected design space. From the \emph{undirected} hypergraph-learning literature, we include the spectral method HGNN~\citep{feng2019hypergraph} and the spatial methods HyperGCN~\citep{yadati2019hypergcnnewmethodtraining} and PhenomNN~\citep{wang2023hypergraph}. From the \emph{directed} hypergraph-learning literature, we consider the spectral method GeDi-HNN~\citep{fiorini2024let}. For each HNN, we compare its original formulation (\textsc{Vanilla}) against the same backbone augmented with PANDA, isolating the benefit of integrating our plug-in without changing the underlying architecture.

\paragraph{Datasets.}
For evaluation, we use eight real-world datasets and follow the pre-processing procedure introduced by~\cite{tran2022directedhypergraphneuralnetwork} and \cite{fiorini2024let} to obtain (directed) hypergraph datasets for node classification. The considered datasets include:
\texttt{Telegram}~\citep{bovet2020activity},
\texttt{Citeseer},
\texttt{PubMed}~\citep{sen2008collective},
\texttt{Cornell},
\texttt{Texas}~\citep{pei2020geomgcngeometricgraphconvolutional},
\texttt{Squirrel},
\texttt{Roman-empire}~\citep{platonov2023a}
and \texttt{email-Enron}~\citep{benson-simplicial}.
Model performance is measured in terms of classification accuracy. Following the standard practice in the literature \citep{allset,wang2022equivariant,fiorini2024let}, we adopt a 50\%/25\%/25\% split for training, validation, and testing, respectively, and, for each model, we report the average test accuracy and the standard deviation over 10 independent runs. Details on the baselines, hyperparameter tuning, and the experimental setup are provided in Appendix~\ref{app:details}.

\paragraph{Integrating PANDA into existing architectures.}
PANDA is designed as a modular plug-in that can be integrated into different backbones. We detail each integration below, starting from GeDi-HNN, which follows the two-stage realisation derived in Section~\ref{sec:gedi-complex}.

HGNN~\citep{feng2019hypergraph} can be written as a degree-normalised node$\rightarrow$hyperedge$\rightarrow$node propagation through the incidence matrix, so we apply the main and auxiliary paths directly to the two incidence aggregations: nodes contribute to hyperedges through the $\textcolor{cyan}{\hat{\alpha}^{\text{main},\ell}_{e,v}}$ and $\textcolor{red}{\hat{\alpha}^{\text{aux},\ell}_{e,v}}$ coefficients, and hyperedges return information to nodes through the $\textcolor{orange}{\check{\alpha}^{\text{main},\ell}_{e,v}}$ and $\textcolor{blue}{\check{\alpha}^{\text{aux},\ell}_{e,v}}$ coefficients, while the original degree and hyperedge-weight normalisations are retained inside the HGNN aggregation.
Note that, since \textsc{PANDA w/o Prior,Aux} is simply the backbone architecture, but with the vanilla aggregation replaced by an attentional one, \textsc{PANDA w/o Prior,Aux} on the HGNN backbones has an identical architecture to HCHA w/ attention~\citep{bai2021hypergraph}.

HyperGCN~\citep{yadati2019hypergcnnewmethodtraining} replaces each hyperedge by a sparse graph built from representative node pairs and mediators before applying a GCN. We keep the hyperedge-to-graph reduction fixed and apply PANDA to the messages on the selected support: the main path re-weights the selected messages, and the auxiliary path adds parallel messages without altering the mediator construction. 

Lastly, PhenomNN~\citep{wang2023hypergraph} is derived from hypergraph-regularised energy functions rather than from a plain convolutional layer. It can be viewed as running message passing through the clique-expanded and star-expanded graphs, separately, and then applying energy function inspired-regularisations to form the new node embeddings. We apply the same rule to the hyperedge-local messages used by its iterative update, treating the original message as the main path and adding a separately transformed incidence-wise message as the auxiliary path before the model's original regularisation and update.

\subsection{Results}
\label{sec:results}

\begin{table}[htb!]
\centering
\small
\caption{Node classification results reporting mean accuracies $\pm$ standard deviation across all datasets. We compare our PANDA plug-in and its ablations with different backbone architectures.}
\label{tab:eval-merged}
\setlength{\tabcolsep}{3pt}
\renewcommand{\arraystretch}{0.95}
\resizebox{\textwidth}{!}{%
\begin{tabular}{c l c c c c c c c c}
\toprule
& \bfseries Variant
& \texttt{Telegram} 
& \texttt{Citeseer}
& \texttt{Pubmed} 
& \texttt{Roman-empire} 
& \texttt{Cornell} 
& \texttt{Texas}
& \texttt{Squirrel}
& \texttt{email-Enron} \\
\midrule

\multirow{5}{*}{\rotatebox[origin=c]{90}{GeDi}}
& VANILLA
& 80.96 $\pm$ 4.506 & 76.47 $\pm$ 1.315 & 89.80 $\pm$ 0.408 & 84.10 $\pm$ 0.399
& 74.32 $\pm$ 5.019 & 82.97 $\pm$ 2.717 & 42.35 $\pm$ 2.059 & 53.24 $\pm$ 7.649 \\

& PANDA w/o Prior,Aux
& 73.27 $\pm$ 6.631 & 75.89 $\pm$ 0.973 & 89.68 $\pm$ 0.415 & 88.65 $\pm$ 0.753
& 78.65 $\pm$ 3.908 & 83.51 $\pm$ 4.432 & 38.62 $\pm$ 1.851 & 58.11 $\pm$ 7.376 \\

& PANDA w/o Aux
& 80.00 $\pm$ 4.485 & 76.20 $\pm$ 0.768 & 90.04 $\pm$ 0.316 & 87.98 $\pm$ 0.758
& 76.22 $\pm$ 3.585 & 84.05 $\pm$ 6.216 & 39.62 $\pm$ 2.573 & 60.81 $\pm$ 9.145 \\

& PANDA
& 91.15 $\pm$ 3.567 & 76.12 $\pm$ 0.921 & 89.32 $\pm$ 0.416 & 90.11 $\pm$ 0.699
& 77.03 $\pm$ 4.230 & 84.05 $\pm$ 4.595 & 42.76 $\pm$ 1.742 & 60.27 $\pm$ 8.466 \\

\cmidrule{2-10}

& $\Delta$(PANDA - Vanilla)
& \textcolor{blue}{+10.19} & \textcolor{red}{-0.35} & \textcolor{red}{-0.48} & \textcolor{blue}{+6.01}
& \textcolor{blue}{+2.71} & \textcolor{blue}{+1.08} & \textcolor{blue}{+0.41} & \textcolor{blue}{+7.03} \\

\midrule\midrule

\multirow{5}{*}{\rotatebox[origin=c]{90}{HGNN}}

& VANILLA
& 58.08 $\pm$ 6.760 & 77.49 $\pm$ 1.226 & 90.08 $\pm$ 0.318 & 78.84 $\pm$ 0.650
& 71.35 $\pm$ 4.713 & 81.89 $\pm$ 5.135 & 39.08 $\pm$ 2.861 & 42.43 $\pm$ 4.366 \\

& PANDA w/o Prior,Aux
& 58.08 $\pm$ 8.104 & 77.08 $\pm$ 1.394 & 89.88 $\pm$ 0.595 & 76.82 $\pm$ 0.405
& 73.78 $\pm$ 5.135 & 82.43 $\pm$ 4.872 & 38.01 $\pm$ 3.601 & 48.11 $\pm$ 5.771 \\

& PANDA w/o Aux
& 71.73 $\pm$ 3.224 & 77.56 $\pm$ 1.070 & 90.02 $\pm$ 0.365 & 82.66 $\pm$ 0.871
& 72.97 $\pm$ 4.009 & 83.24 $\pm$ 4.490 & 39.68 $\pm$ 2.830 & 44.05 $\pm$ 6.954 \\

& PANDA
& 72.12 $\pm$ 2.885 & 77.17 $\pm$ 1.125 & 89.77 $\pm$ 0.368 & 82.37 $\pm$ 0.511
& 74.59 $\pm$ 2.162 & 83.24 $\pm$ 5.644 & 40.43 $\pm$ 0.999 & 54.59 $\pm$ 8.090 \\

\cmidrule{2-10}

& $\Delta$(PANDA - Vanilla)
& \textcolor{blue}{+14.04} & \textcolor{red}{-0.32} & \textcolor{red}{-0.21} & \textcolor{blue}{+3.53} 
& \textcolor{blue}{+3.24} & \textcolor{blue}{+1.31} & \textcolor{blue}{+1.35} & \textcolor{blue}{+12.16} \\

\midrule\midrule

\multirow{5}{*}{\rotatebox[origin=c]{90}{HyperGCN}}

& VANILLA
& 61.92 $\pm$ 4.106 & 75.59 $\pm$ 0.776 & 88.94 $\pm$ 0.326 & 65.04 $\pm$ 0.702
& 57.57 $\pm$ 7.257 & 65.41 $\pm$ 6.373 & 38.28 $\pm$ 1.628 & 47.30 $\pm$ 8.653 \\

& PANDA w/o Prior,Aux
& 62.50 $\pm$ 7.410 & 75.46 $\pm$ 0.732 & 87.84 $\pm$ 0.244 & 77.12 $\pm$ 0.486
& 67.03 $\pm$ 9.269 & 76.49 $\pm$ 8.967 & 34.51 $\pm$ 2.694 & 45.75 $\pm$ 2.936 \\

& PANDA w/o Aux
& 68.65 $\pm$ 5.374 & 75.41 $\pm$ 0.753 & 88.68 $\pm$ 0.499 & 77.94 $\pm$ 0.419
& 67.03 $\pm$ 6.139 & 73.62 $\pm$ 6.397 & 34.67 $\pm$ 2.126 & 55.68 $\pm$ 7.852 \\

& PANDA
& 70.58 $\pm$ 5.165 & 75.63 $\pm$ 0.894 & 88.59 $\pm$ 0.568 & 80.16 $\pm$ 0.661
& 64.86 $\pm$ 6.620 & 78.65 $\pm$ 5.976 & 42.46 $\pm$ 2.923 & 54.86 $\pm$ 6.402 \\

\cmidrule{2-10}

& $\Delta$(PANDA - Vanilla)
& \textcolor{blue}{+8.66} & \textcolor{blue}{+0.04} & \textcolor{red}{-0.25} & \textcolor{blue}{+15.12} 
& \textcolor{blue}{+7.29} & \textcolor{blue}{+13.24} & \textcolor{blue}{+4.18} & \textcolor{blue}{+7.56} \\

\midrule\midrule

\multirow{5}{*}{\rotatebox[origin=c]{90}{PhenomNN}}

& VANILLA
& 44.81 $\pm$ 4.472 & 77.40 $\pm$ 1.105 & 90.24 $\pm$ 0.365 & 80.33 $\pm$ 0.423
& 74.05 $\pm$ 4.713 & 82.97 $\pm$ 4.196 & 37.63 $\pm$ 2.451 & 39.43 $\pm$ 3.148 \\

& PANDA w/o Prior,Aux
& 52.12 $\pm$ 8.088 & 77.52 $\pm$ 1.112 & 90.10 $\pm$ 0.392 & 78.72 $\pm$ 0.492
& 76.22 $\pm$ 4.151 & 80.81 $\pm$ 6.886 & 38.28 $\pm$ 2.841 & 34.05 $\pm$ 8.470 \\

& PANDA w/o Aux
& 75.58 $\pm$ 6.084 & 77.20 $\pm$ 1.075 & 87.87 $\pm$ 0.666 & 76.63 $\pm$ 0.376
& 75.14 $\pm$ 2.648 & 82.43 $\pm$ 4.230 & 38.49 $\pm$ 2.593 & 48.11 $\pm$ 6.817 \\

& PANDA
& 75.82 $\pm$ 4.343 & 77.08 $\pm$ 1.113 & 87.51 $\pm$ 0.581 & 85.08 $\pm$ 0.752
& 75.68 $\pm$ 2.961 & 82.70 $\pm$ 5.565 & 41.01 $\pm$ 1.999 & 50.27 $\pm$ 6.304 \\

\cmidrule{2-10}

& $\Delta$(PANDA - Vanilla)
& \textcolor{blue}{+31.01} & \textcolor{red}{-0.32} & \textcolor{red}{-2.73} & \textcolor{blue}{+4.75}
& \textcolor{blue}{+1.63} & \textcolor{red}{-0.27} & \textcolor{blue}{+3.38} & \textcolor{blue}{+10.84} \\

\bottomrule
\end{tabular}
}
\end{table}

Table~\ref{tab:eval-merged} reports node classification accuracy across eight benchmark datasets for PANDA integrated with four HNN models: GeDi-HNN, HGNN, HyperGCN, and PhenomNN. For each model, we compare the original formulation (\textsc{Vanilla}), two PANDA ablations, and the full PANDA module. Overall, PANDA achieves an average improvement of $5.18$ percentage points over the corresponding \textsc{Vanilla} models. The gains concentrate on heterophilic benchmarks, while performance remains largely unchanged on homophilic citation datasets. The largest improvements occur on Telegram, where PANDA lifts all four models by near double-digit margins, reaching $31.01$ percentage points on PhenomNN and exceeding $10$ points on GeDi-HNN and HGNN. Heterophilic hypergraph datasets follow the same pattern: HyperGCN benefits the most, with gains of over $13$ percentage points on both Roman-empire and Texas, while the other models show consistent improvements. On email-Enron, PANDA improves all four models by an average of $9.40$ percentage points, suggesting that the gains are driven by the underlying interaction structure rather than by a single model family.
On the homophilic citation datasets Citeseer and Pubmed, where the structural prior of the \textsc{Vanilla} models is already well aligned with the data, the difference between full PANDA and \textsc{Vanilla} remains within $[-2.73,+0.04]$ percentage points, with most changes below one point in absolute value. PANDA therefore acts as a corrective plug-in when fixed structural propagation is misaligned with the data, while preserving performance when the original propagation is already appropriate.

\textbf{Ablations.}
The ablation rows isolate the effect of PANDA's two main components. \textsc{PANDA} w/o Prior,Aux keeps only the learned receiver-normalised attention. This variant is sometimes unable to recover the \textsc{Vanilla} accuracy and does not consistently outperform it, suggesting that learned attention alone is insufficient to replace the fixed structural propagation of existing HNN models. Adding the structural prior (\textsc{PANDA} w/o Aux) recovers much of the \textsc{Vanilla} performance, while also allowing the model to outperform the baseline on some datasets, highlighting the importance of both attention and the structural priors. The full PANDA module further adds the gated, non-competitive auxiliary path, which is decisive in several cases where the main path alone remains limited. This indicates that the auxiliary path captures complementary information that the receiver-normalised main path cannot express on its own.

\begin{table}[!tb]
\centering
\small
\caption{Node classification results reporting mean accuracies $\pm$ standard deviation across all datasets for GCN, GAT and PANDA on GCN.}
\label{tab:eval-gnn}
\setlength{\tabcolsep}{3pt}
\renewcommand{\arraystretch}{0.95}
\resizebox{\textwidth}{!}{%
\begin{tabular}{c l c c c c c c c c}
\toprule
& \bfseries Variant
& \texttt{Telegram} 
& \texttt{Citeseer}
& \texttt{Pubmed} 
& \texttt{Roman-empire} 
& \texttt{Cornell} 
& \texttt{Texas}
& \texttt{Squirrel}
& \texttt{email-Enron} \\
\midrule

\multirow{4}{*}{\rotatebox[origin=c]{90}{GCN}}
& VANILLA
& 90.38 $\pm$ 2.852 & 77.00 $\pm$ 1.197 & 89.28 $\pm$ 0.272 & 75.18 $\pm$ 0.683
& 52.43 $\pm$ 9.300  & 67.57 $\pm$ 6.396 & 43.70 $\pm$ 1.947 & 62.16 $\pm$ 5.539 \\

& GAT (PANDA w/o Prior,Aux)
& 82.88 $\pm$ 3.488 & 76.75 $\pm$ 0.919 & 89.92 $\pm$ 0.247 & 85.90 $\pm$ 0.707
& 53.78 $\pm$ 10.429 & 68.92 $\pm$ 6.648 & 39.73 $\pm$ 1.634 & 68.92 $\pm$ 6.310 \\

& PANDA
& 91.73 $\pm$ 3.851 & 76.30 $\pm$ 1.070 & 88.97 $\pm$ 0.466 & 86.74 $\pm$ 0.711
& 73.78 $\pm$ 6.049 & 80.54 $\pm$ 4.804 & 42.84 $\pm$ 1.322 & 64.86 $\pm$ 7.151 \\

\cmidrule{2-10}

& $\Delta$(PANDA - Vanilla)
& \textcolor{blue}{+1.37} & \textcolor{red}{-0.70} & \textcolor{red}{-0.31} & \textcolor{blue}{+11.55}
& \textcolor{blue}{+21.35} & \textcolor{blue}{+12.97} & \textcolor{red}{-0.86} & \textcolor{blue}{+2.70} \\

\bottomrule
\end{tabular}
}
\end{table}

\textbf{PANDA on GNNs.} We have presented PANDA with a focus on directed hypergraphs as they subsume undirected hypergraphs and ordinary graphs. \Cref{tab:eval-gnn} shows that PANDA's gains translate to ordinary GNNs as well. We see that PANDA consistently improves performance for GCN across the datasets, with the exceptions where the differences are primarily small. PANDA + GCN also consistently outperforms GAT, with exceptions on the homophilic datasets where the difference is small, and the email-Enron dataset, where the difference is greater.

\begin{table}[!tb]
\centering
\caption{Main auxiliary-path analysis across five datasets. 
\(\mathcal H_{\mathrm{CE}}\) is clique-expansion homophily from \cite{wang2023hypergraph}.
\(\Delta\)Acc.\ is the accuracy change after shutting off both auxiliary
gates, $\textcolor{red}{\hat{g}^\ell_{e,v}}$ and $\textcolor{blue}{\check{g}^\ell_{u,e}}$. 
\(q_{\mathrm{mis}}\) is the fraction of test nodes whose
clique-expansion neighbour-majority class differs from the true label.
\(q_{\mathrm{help}}\) and \(q_{\mathrm{hurt}}\) are the fractions of nodes
flipped onto or off the true class by the auxiliary path.
\(q_{\mathrm{nbr\text{-}rescue}}\) is the fraction of test
nodes rescued specifically from the misleading neighbour-majority class.
\(\overline{\Delta}^{\mathrm{nbr}}_{\neq}\) is the true-vs-neighbour-majority
logit-margin shift on misleading-neighbour nodes. All rates are percentages;
full definitions are in Appendix~\ref{sec:analysis-app}.}
\label{tab:main-aux-analysis}
\setlength{\tabcolsep}{4pt}
\begin{tabular*}{\textwidth}{@{\extracolsep{\fill}}lrrrrrr@{}}
    \toprule
    Dataset
    & \(\mathcal H_{\mathrm{CE}}\)
    & \(\Delta\)Acc.
    & \(q_{\mathrm{mis}}\)
    & \makecell{\(q_{\mathrm{help}}/ q_{\mathrm{hurt}}\)}
    & \(q_{\mathrm{nbr\text{-}rescue}}\)
    & \(\overline{\Delta}^{\mathrm{nbr}}_{\neq}\) \\
    \midrule
    \texttt{Roman-empire}  & 0.236 & {\(-27.7\)} & 95.68 & 29.44 / 1.77  & 2.20  & \(+5.29\) \\
    \texttt{Telegram}      & 0.285 & \(-54.0\) & 81.73 & 56.54 / 2.50  & 16.35 & \(+3.58\) \\
    \texttt{Cornell}       & 0.420 & \(-4.6\)  & 57.84 & 5.95 / 1.35   & 1.08  & \(+0.09\) \\
    \texttt{Texas}         & 0.505 & \(+0.0\)  & 47.03 & 3.24 / 3.24   & 1.35  & \(+0.09\) \\
    \texttt{Citeseer}      & 0.795 & \(-1.0\)  & 23.97 & 1.84 / 0.82   & 0.22  & \(-0.28\) \\
    \bottomrule
\end{tabular*}
    \vspace{-0.4cm}
\end{table}

\textbf{Auxiliary-Path Analysis.}
We further analyse the learned auxiliary path of the GeDi-PANDA on five datasets using
post-hoc inference-time interventions on fixed trained checkpoints; full definitions, routing statistics, and additional geometric diagnostics are reported in Appendix~\ref{sec:analysis-app}. These interventions measure the predictive dependence of the trained model on the auxiliary gates at test time. Table~\ref{tab:main-aux-analysis}
summarises the main findings. Disabling both $\textcolor{red}{\hat{g}^\ell_{e,v}}$ and $\textcolor{blue}{\check{g}^\ell_{u,e}}$ auxiliary gates has a large
effect on Telegram and Roman-empire, decreasing test accuracy by \(54.0\) and
\(27.7\) percentage points, respectively, with a smaller effect on Citeseer, Cornell, and Texas. 

To understand whether this effect corresponds to correcting misleading neighbourhood information, we measure \(\overline{\Delta}_{\neq}^{\mathrm{nbr}}\), i.e., how much the auxiliary path shifts the logit margin between the true class and the neighbour-majority class, when the two are different. On the two low-homophily datasets, as reported by \cite{wang2023hypergraph},
Roman-empire and Telegram, the auxiliary path strongly increases the
true-vs-neighbour-majority logit margin on misleading-neighbour nodes
(\(+5.29\) and \(+3.58\)). Thus, when information flowing from
neighbourhoods is unreliable, the auxiliary path behaves as a counter-neighbourhood signal. However, strict neighbour-majority rescues explain only part of the total helped flips: \(2.20\%\) out of \(29.44\%\) on
Roman-empire and \(16.35\%\) out of \(56.54\%\) on Telegram. This indicates
that the auxiliary path is not merely correcting predictions that follow the
wrong neighbour-majority class; rather, it improves margins against a broader
set of competing classes. 

The cosine analysis in Appendix~\ref{sec:analysis-cos-pass} provides a
complementary geometric view. We measure the source- and receiver-wise cosine similarity between the
main-path representation and the auxiliary representation in each pass. In the
\(\VE\) pass, the cosine is positive across datasets, indicating that the
auxiliary node representation usually retains part of the main node direction. On the other hand, in the \(\EV\) pass, the cosine is instead close to zero, showing that the
auxiliary-transformed hyperedge representation is not a rescaled copy of the main one but
a distinct, approximately orthogonal direction. This geometric evidence does
not by itself explain the margin gains, but it supports the interpretation that
the auxiliary branch provides a non-redundant message channel rather than
simply amplifying the main path.

\textbf{Non-Competitiveness Ablation.} We further investigate whether the performance gains come from the non-competitive nature of the auxiliary path or from the additional parameters introduced in the auxiliary path.
Towards this end, we compare the performance of PANDA to a variant where we normalise the auxiliary path's scores, while maintaining identical parameter counts. We dub this variant Norm Aux. \Cref{tab:norm-aux-eval} shows that PANDA consistently outperforms the Norm Aux ablation, with a few small exceptions where the gap is primarily small. This strongly supports our hypothesis that non-competitiveness is important.

\begin{table}[tbp!]
\centering
\small
\caption{Node classification results reporting mean accuracies $\pm$ standard deviation across all datasets. We compare our PANDA plug-in with the Norm Aux variant on different backbone architectures.}
\label{tab:norm-aux-eval}
\setlength{\tabcolsep}{3pt}
\renewcommand{\arraystretch}{0.95}
\resizebox{\textwidth}{!}{%
\begin{tabular}{l c c c c c c c c}
\toprule
\bfseries Model
& \texttt{Telegram} 
& \texttt{Citeseer}
& \texttt{Pubmed} 
& \texttt{Roman-empire} 
& \texttt{Cornell} 
& \texttt{Texas}
& \texttt{Squirrel}
& \texttt{email-Enron} \\
\midrule

GeDi + Norm Aux & 75.96 $\pm$ 5.319 & 76.56 $\pm$ 0.862 & 89.25 $\pm$ 0.498 & 89.42 $\pm$ 0.466 & 75.14 $\pm$ 3.784 & 80.27 $\pm$ 6.401 & 40.07 $\pm$ 4.317 & 52.16 $\pm$ 7.649 \\
GeDi + PANDA & 91.15 $\pm$ 3.567 & 76.12 $\pm$ 0.921 & 89.32 $\pm$ 0.416 & 90.11 $\pm$ 0.699 & 77.03 $\pm$ 4.230 & 84.05 $\pm$ 4.595 & 42.76 $\pm$ 1.742 & 60.27 $\pm$ 8.466 \\
\midrule
$\Delta$(PANDA - Norm Aux) & \textcolor{blue}{+15.19} & \textcolor{red}{-0.44} & \textcolor{blue}{+0.07} & \textcolor{blue}{+0.69} & \textcolor{blue}{+1.89} & \textcolor{blue}{+3.78} & \textcolor{blue}{+2.69} & \textcolor{blue}{+8.11} \\
\midrule
\midrule

HGNN + Norm Aux & 69.62 $\pm$ 3.922 & 77.38 $\pm$ 0.648 & 89.66 $\pm$ 0.379 & 79.83 $\pm$ 0.621 & 70.54 $\pm$ 4.265 & 80.54 $\pm$ 2.911 & 38.78 $\pm$ 2.902 & 51.62 $\pm$ 3.354 \\
HGNN + PANDA & 72.12 $\pm$ 2.885 & 77.17 $\pm$ 1.125 & 89.77 $\pm$ 0.368 & 82.37 $\pm$ 0.511 & 74.59 $\pm$ 2.162 & 83.24 $\pm$ 5.644 & 40.43 $\pm$ 0.999 & 54.59 $\pm$ 8.090 \\
\midrule
$\Delta$(PANDA - Norm Aux) & \textcolor{blue}{+2.50} & \textcolor{red}{-0.21} & \textcolor{blue}{+0.11} & \textcolor{blue}{+2.54} & \textcolor{blue}{+4.05} & \textcolor{blue}{+2.70} & \textcolor{blue}{+1.65} & \textcolor{blue}{+2.97} \\
\midrule
\midrule

HyperGCN + Norm Aux & 67.31 $\pm$ 5.160 & 75.87 $\pm$ 1.118 & 88.28 $\pm$ 0.218 & 78.67 $\pm$ 0.581 & 66.49 $\pm$ 5.822 & 77.57 $\pm$ 7.257 & 37.83 $\pm$ 3.124 & 45.59 $\pm$ 3.713 \\
HyperGCN + PANDA & 70.58 $\pm$ 5.165 & 75.63 $\pm$ 0.894 & 88.59 $\pm$ 0.568 & 80.16 $\pm$ 0.661 & 64.86 $\pm$ 6.620 & 78.65 $\pm$ 5.976 & 42.46 $\pm$ 2.923 & 54.86 $\pm$ 6.402 \\
\midrule
$\Delta$(PANDA - Norm Aux) & \textcolor{blue}{+3.27} & \textcolor{red}{-0.24} & \textcolor{blue}{+0.31} & \textcolor{blue}{+1.49} & \textcolor{red}{-1.63} & \textcolor{blue}{+1.08} & \textcolor{blue}{+4.63} & \textcolor{blue}{+9.27} \\
\midrule
\midrule

PhenomNN + Norm Aux & 72.88 $\pm$ 3.889 & 77.04 $\pm$ 1.259 & 88.56 $\pm$ 0.378 & 79.75 $\pm$ 0.419 & 74.95 $\pm$ 4.090 & 81.97 $\pm$ 4.018 & 39.57 $\pm$ 3.157 & 43.76 $\pm$ 8.683 \\
PhenomNN + PANDA & 75.82 $\pm$ 4.343 & 77.08 $\pm$ 1.113 & 87.51 $\pm$ 0.581 & 85.08 $\pm$ 0.752 & 75.68 $\pm$ 2.961 & 82.70 $\pm$ 5.565 & 41.01 $\pm$ 1.999 & 50.27 $\pm$ 6.304 \\
\midrule
$\Delta$(PANDA - Norm Aux) & \textcolor{blue}{+2.94} & \textcolor{blue}{+0.04} & \textcolor{red}{-1.05} & \textcolor{blue}{+5.33} & \textcolor{blue}{+0.73} & \textcolor{blue}{+0.73} & \textcolor{blue}{+1.44} & \textcolor{blue}{+6.51} \\
\bottomrule
\end{tabular}
}
\end{table}

\textbf{Additional Analysis.} We perform parameter and inference time comparisons for PANDA and the vanilla models in
\Cref{sec:appendix-overheads}, and these demonstrate that PANDA can improve parameter efficiency and, in several settings,
inference speed. We also extend our analysis on additional datasets in \Cref{sec:appendix-additional-eval} and show
that similar trends follow on these additional datasets. We include a large scale hypergraph dataset and show how
PANDA helps with scalability, by allowing for more parameter efficient and faster models. We also perform norm stability
analyses in \Cref{sec:appendix-stability} to ensure there are no stability related issues with the introduction of the
auxiliary paths in PANDA.

\section{Conclusions}
\label{sec:conclusions}
HNNs typically rely either on fixed structural weights, which can propagate misleading information in heterophilic settings, or on receiver-wise normalised attention, whose coefficients are competitive by construction.
We introduced PANDA, a plug-in for two-stage hypergraph message-passing layers, applicable to both undirected and directed HNNs. Its main path interpolates between learned receiver-normalised attention and structural incidence priors, while a non-competitive auxiliary path supplies an additional channel that the receiver-normalised path cannot express on its own.
We further showed that GeDi-HNN's complex Hermitian convolution admits a two-stage realisation through the original complex incidence matrix and its conjugate, but that this realisation is not equivalent to applying a Hermitian graph operator after a directed star expansion of the hypergraph.
Across four HNNs, PANDA yields consistent improvements on heterophilic benchmarks without degrading performance on homophilic ones, acting as a corrective plug-in where the structural propagation is misaligned with the data. Our post-hoc analysis makes this behaviour explicit. Removing the auxiliary branch at test time sharply reduces accuracy, showing that it is informative for the trained model. On heterophilic datasets, the auxiliary branch shifts predictions away from misleading neighbourhood signals towards the correct class, and its representations, especially in the hyperedge-to-node stage, are nearly orthogonal to those of the main path, indicating a distinct message channel rather than a rescaling.
Future work will address the main architectural limitation of PANDA, i.e., being restricted to two-stage message-passing HNNs, by extending it to more general information-propagation paradigms rooted in algebraic topology, such as sheaf neural networks.

\begin{ack}
  This work is partially funded by Trinity College, Cambridge.
\end{ack}

% \clearpage

\bibliography{bibliography}
\bibliographystyle{plainnat}

%%%%%%%%%%%%%%%%%%%%%%%%%%%%%%%%%%%%%%%%%%%%%%%%%%%%%%%%%%%%
\appendix
\section{Proof of our theoretical results}
\label{sec:appendix-proofs}

\begin{theorem}\label{thm:gedi-recovery-appendix}
\Paste{cp-thm:gedi-recovery}
\end{theorem}
\begin{proof}
  GeDi defines the Signless Laplacian and (Signed) Laplacian operators as
  \begin{equation}
    \vec{Q}_N
    =
    D_v^{-1/2}\,\vec{B}\,W\,D_e^{-1}\,\vec{B}^*\,D_v^{-1/2},
    \qquad
    \vec{L}_N = I-\vec{Q}_N,
    \label{eq:gedi-q-l-proof}
  \end{equation}
  with $\vec{B}$ defined as in Equation~\eqref{eq:B}.
  Its convolutional layer maps $X^{(\ell)}\in\mathbb{C}^{n\times c}$ to
  \begin{equation}
    X^{(\ell+1)}_{\mathrm{prior}}
    =
    \phi\!\left(X^{(\ell)}\Theta_0^G+\vec{L}_N X^{(\ell)}\Theta_1^G\right),
    \label{eq:gedi-conv-layer}
  \end{equation}
  for learnable matrices $\Theta_0^G,\Theta_1^G$.

  Given that $*$ denotes the complex conjugate-adjoint operator, under the hypotheses of the theorem, $\hat\alpha^{\mathrm{main}}=\hat\alpha^{\mathrm{prior}}$, $\check\alpha^{\mathrm{main}}=\check\alpha^{\mathrm{prior}}$, and the auxiliary path is switched off. Therefore the node-to-hyperedge aggregation in
  Eq.~\eqref{eq:gadh-h} becomes, for every $e \in \mathcal{E}$,
  \begin{align*}
    \mathbf{h}^{(\ell+1)}_e
     & = \sum_{v \in H(e)} \sqrt{\tfrac{w_e}{\delta_e d_v}}\, \mathbf{x}^{(\ell)}_v
    \;+\; i \sum_{v \in T(e)} \sqrt{\tfrac{w_e}{\delta_e d_v}}\, \mathbf{x}^{(\ell)}_v
    = \sum_{v \in \mathcal{V}} \sqrt{\tfrac{w_e}{\delta_e d_v}}\, \vec{B}^*_{e,v}\, \mathbf{x}^{(\ell)}_v.
  \end{align*}
  Stacking across hyperedges into $H^{(\ell+1)} \in \mathbb{C}^{m \times c}$,
  \begin{equation}
    H^{(\ell+1)} \;=\; D_e^{-1/2}\, W^{1/2}\, \vec{B}^*\, D_v^{-1/2}\, X^{(\ell)}.
    \label{eq:proof-stage1}
  \end{equation}
  The hyperedge-to-node aggregation in Eq.~\eqref{eq:gadh-xprime} similarly gives
  \begin{align*}
    \mathbf{x}^{\prime(\ell+1)}_u
     & = \sum_{e \in \Gamma_H(u)} \sqrt{\tfrac{w_e}{d_u \delta_e}}\, \mathbf{h}^{(\ell+1)}_e
    \;-\; i \sum_{e \in \Gamma_T(u)} \sqrt{\tfrac{w_e}{d_u \delta_e}}\, \mathbf{h}^{(\ell+1)}_e
    = \sum_{e \in \mathcal{E}} \sqrt{\tfrac{w_e}{d_u \delta_e}}\, \vec{B}_{u,e}\, \mathbf{h}^{(\ell+1)}_e,
  \end{align*}
  which in matrix form reads
  \begin{equation}
    X^{\prime(\ell+1)} \;=\; D_v^{-1/2}\, \vec{B}\, W^{1/2}\, D_e^{-1/2}\, H^{(\ell+1)}.
    \label{eq:proof-stage2}
  \end{equation}
  Composing Eqs.~\eqref{eq:proof-stage1} and~\eqref{eq:proof-stage2} and using that
  $W, D_e, D_v$ are real diagonal,
  \begin{align*}
    X^{\prime(\ell+1)}
     & = D_v^{-1/2}\, \vec{B}\, W^{1/2}\, D_e^{-1/2}\, D_e^{-1/2}\, W^{1/2}\, \vec{B}^*\, D_v^{-1/2}\, X^{(\ell)} \\
     & = D_v^{-1/2}\, \vec{B}\, W\, D_e^{-1}\, \vec{B}^*\, D_v^{-1/2}\, X^{(\ell)}
    \;=\; \vec{Q}_N\, X^{(\ell)},
  \end{align*}
  where the last equality is exactly the definition of $\vec{Q}_N$ in
  Eq.~\eqref{eq:gedi-q-l-proof}. Hence the proposed aggregation mechanism recovers
  the GeDi convolution operator $\vec{Q}_N$. Combining this with the update in (iii)
  yields
  \begin{equation*}
    X^{(\ell+1)} \;=\; \phi\!\left(X^{(\ell)} \Theta_0 \;+\; \vec{Q}_N\, X^{(\ell)} \Theta_1\right).
  \end{equation*}
  Given any GeDi parameters $(\Theta_0^G,\Theta_1^G)$, choose
  $\Theta_0=\Theta_0^G+\Theta_1^G$ and $\Theta_1=-\Theta_1^G$. Then
  \begin{align*}
    \phi\!\left(X^{(\ell)} \Theta_0 + \vec{Q}_N X^{(\ell)}\Theta_1\right)
     & =
    \phi\!\left(X^{(\ell)}(\Theta_0^G+\Theta_1^G)
    - \vec{Q}_N X^{(\ell)}\Theta_1^G\right)    \\
     & =
    \phi\!\left(X^{(\ell)}\Theta_0^G
    + (I-\vec{Q}_N)X^{(\ell)}\Theta_1^G\right) \\
     & =
    \phi\!\left(X^{(\ell)}\Theta_0^G+\vec{L}_N X^{(\ell)}\Theta_1^G\right),
  \end{align*}
  which is the GeDi-HNN layer in Eq.~\eqref{eq:gedi-conv-layer}.
\end{proof}

\begin{theorem}\label{thm:gedi-star-expansion-appendix}
\Paste{cp-thm:gedi-star-expansion}
\end{theorem}
\begin{proof}
  One may ask whether the same operator can be obtained by replacing each directed
  hyperedge $e$ with an auxiliary node $h_e$, connecting every tail node to $h_e$
  and then connecting $h_e$ to every head node. Formally, define the
  star-expanded directed graph
  \begin{equation*}
    \mathcal{G}_\star=(\mathcal{V}_\star,\mathcal{A}_\star),
    \qquad
    \mathcal{V}_\star=\mathcal{V}\cup\{h_e:e\in\mathcal{E}\},
  \end{equation*}
  with ordinary directed edges
  \begin{equation*}
    \mathcal{A}_\star
    =
    \{a_{t,e}=(t,h_e):t\in T(e)\}
    \cup
    \{a_{e,r}=(h_e,r):r\in H(e)\}.
  \end{equation*}
  Each $a=(p,q)\in\mathcal{A}_\star$ is now an ordinary directed edge, i.e.,
  $T_\star(a)=\{p\}$ and $H_\star(a)=\{q\}$. Let
  $\vec{B}_\star$ be the corresponding GeDi incidence matrix, with
  $(\vec{B}_\star)_{p,a}=-i$ and $(\vec{B}_\star)_{q,a}=1$, and let
  $W_\star$, $D_{\mathcal{V}_\star}$, and $D_{\mathcal{A}_\star}$ be the edge
  weights, vertex degrees, and edge degrees of this ordinary directed graph. The
  GeDi adjacency and Laplacian on the star expansion are
  \begin{equation*}
    \vec{Q}_\star
    =
    D_{\mathcal{V}_\star}^{-1/2}
    \vec{B}_\star W_\star D_{\mathcal{A}_\star}^{-1}
    \vec{B}_\star^*
    D_{\mathcal{V}_\star}^{-1/2},
    \qquad
    \vec{L}_\star=I-\vec{Q}_\star .
  \end{equation*}
  Using $\vec{L}_\star$ or $\vec{Q}_\star$ on this expanded graph is not
  equivalent to the hypergraph GeDi operator $\vec{Q}_N$. A single GeDi step on
  the star-expanded graph only mixes original nodes with auxiliary nodes, so it
  does not create the off-diagonal interactions between distinct original nodes
  that lie in the same hyperedge. If one instead compares the two-step adjacency
  operator $\vec{Q}_\star^2$, the complex phases still do not match.

  Ignoring positive degree and weight normalisations, the phase of the hypergraph
  GeDi contribution from input node $v$ to output node $u$ through hyperedge $e$ is
  \[
    c_e(u,v)=\vec{B}_{u,e}\vec{B}^*_{e,v}
    =
    \begin{cases}
      1,  & u,v\in H(e)\text{ or }u,v\in T(e), \\
      i,  & u\in H(e),\, v\in T(e),            \\
      -i, & u\in T(e),\, v\in H(e).
    \end{cases}
  \]
  In the star-expanded graph, a two-hop interaction from a tail node
  $v\in T(e)$ to a head node $u\in H(e)$ follows $v\to h_e\to u$. The two ordinary
  directed edges contribute phases $i$ and $i$, giving $i\cdot i=-1$ instead of
  $i$. Similarly, the reverse cross-role path gives $(-i)\cdot(-i)=-1$ instead of
  $-i$. Thus the star expansion changes the cross-role phases.

  The smallest counterexample has two vertices and one directed hyperedge:
  \[
    \mathcal{V}=\{t,r\},
    \qquad
    T(e)=\{t\},
    \qquad
    H(e)=\{r\},
    \qquad
    w_e=1.
  \]
  Then $\delta_e=2$, $d_t=d_r=1$, and the original hypergraph operator has
  \[
    (\vec{Q}_N)_{r,t}
    =
    \frac{1}{2}\vec{B}_{r,e}\vec{B}^*_{e,t}
    =
    \frac{i}{2}.
  \]
  Replace $e$ by the directed path $t\to h_e\to r$, with ordinary directed edges
  $a=(t,h_e)$ and $b=(h_e,r)$. A single GeDi step on this expanded graph has zero
  direct coefficient from $t$ to $r$. The two-step contribution through $h_e$ has
  phase
  \[
    \left(\vec{B}_{r,b}\vec{B}^*_{b,h_e}\right)
    \left(\vec{B}_{h_e,a}\vec{B}^*_{a,t}\right)
    =
    (1\cdot i)(1\cdot i)
    =
    -1.
  \]
  With the same GeDi normalisations on the expanded graph,
  $d_t=d_r=1$, $d_{h_e}=2$, and $\delta_a=\delta_b=2$, hence
  \[
    (\vec{Q}_{\star})_{h_e,t}
    =
    \frac{i}{2\sqrt{2}},
    \qquad
    (\vec{Q}_{\star})_{r,h_e}
    =
    \frac{i}{2\sqrt{2}},
    \qquad
    (\vec{Q}_{\star})_{r,h_e}(\vec{Q}_{\star})_{h_e,t}
    =
    -\frac{1}{8}.
  \]
  This differs from the hypergraph coefficient $i/2$ already in phase.

  The correct equivalence is instead the two-stage computation of the original
  hypergraph operator $\vec{Q}_N$. The symbol $h_e$ can be used as a temporary
  hyperedge variable, but it should not be interpreted as a new vertex of an
  ordinary directed graph. The two stages must use the conjugate incidence
  $\vec{B}^*$ from nodes to hyperedges and the incidence $\vec{B}$ from hyperedges
  back to nodes:
  \begin{equation*}
    m_e
    =
    \sum_{v\in e}\vec{B}^*_{e,v}\frac{x_v}{\sqrt{d_v}},
    \qquad
    (\vec{Q}_N x)_u
    =
    \frac{1}{\sqrt{d_u}}
    \sum_{e\ni u}
    \frac{w_e}{\delta_e}\vec{B}_{u,e}m_e.
  \end{equation*}
  This is exactly the $\vec{B}^*$-then-$\vec{B}$ factorisation of
  $\vec{Q}_N$. Hence the attentive message-passing reduction is equivalent to
  GeDi when the intermediate hyperedge variables are computed with these frozen
  complex coefficients, but not when one first converts the hypergraph into an
  ordinary directed graph and then applies GeDi to that graph.
\end{proof}

\section{Experimental Setup}\label{app:details}

\subsection{Code Repository}

The code written for this research work is attached as supplementary material and will be freely distributed under the Apache 2.0 licence once accepted.\footnote{\url{https://www.apache.org/licenses/LICENSE-2.0}}

The baseline models used in this paper were taken from the publicly available implementations from GeDi-HNN \url{https://github.com/Stefa1994/GeDi-HNN} with Apache License 2.0.

\subsection{Hardware Details}
All experiments are carried out on two different workstations: one equipped with four NVIDIA RTX~2080 GPUs (11~GB each) and Intel Xeon Gold 6248R processors, and another featuring an Intel Xeon CPU E5-2630 and two NVIDIA TITAN Xp GPUs with 12~GB of VRAM each. %We utilized the WandB platform to monitor training procedures.

The runtime required for a given hyperparameter setting, which involves training and validation over 10 different dataset splits, depends on both the dataset and the backbone architecture. For example, on the \texttt{Telegram} dataset, each hyperparameter setting of GeDi+PANDA requires on average 14.2 mins, while Phenomnn+PANDA requires on average 21.3 mins. On the other hand, GeDi+PANDA requires only 7.2 mins on average on the \texttt{Cornell} dataset, while it requires on average 40.4 mins on the \texttt{Roman-empire} dataset. On average, each setting for GeDi+PANDA requires 12.5 mins.

\subsection{Hyperparameter tuning}\label{app:tuning}
We utilise the WandB platform to monitor training procedures and to carry out hyperparameter tuning. All models are trained for up to 500 epochs with early stopping set to 200 epochs. We employ Adam~\citep{kingma2017adammethodstochasticoptimization} for optimising the model parameters with \(\text{lr} \in [0.001, 0.02]\), \(\text{wd}\in\{0,\,5\times10^{-5},5\times10^{-4},5\times10^{-3}\}\). For all the models, we adopt a dropout \(\in\{0.1,0.2,\ldots,0.9\}\), and for each model that has a selectable number of layers for the final classifier we fix it to 2.
For each baseline, we select a range of parameters consistent with those investigated in their respective original works:

\begin{itemize}
  \item HGNN: basic blocks \(\{1, 2, 3, 4, 5\}\); MLP hidden width \(\{64,128,256,512\}\), classifier width \(\{64,128,256\}\).
  \item HyperGCN: basic blocks \(\{1, 2, 3, 4, 5\}\); HyperGCN-Fast \(\in \{\text{True}, \text{False}\}\); Using-Mediators \(\in \{\text{True}, \text{False}\}\).
  \item PhenomNN: hidden width \(\{64,128,256,512\}\); \(\lambda_{0}\in \{0, 0.1, 0.2, \ldots, 100.0\}\); \(\lambda_{1}\in \{0, 0.1, 0.2, \ldots, 100.0\}\); propagation steps \(\{8,16\}\); (H, HisI, twoHgamma) \(\in \{(\text{False}, \text{False}, \text{False}), (\text{True}, \text{False}, \text{False}), (\text{True}, \text{True}, \text{False}), (\text{True}, \text{False}, \text{True})\}\).
  \item GeDi-HNN: convolutional layers \(\{1,2,3,4,5\}\); MLP hidden width \(\{64,128,256,512\}\); classifier width \(\{64,128,256\}\); other complex \(\in \{\text{True}, \text{False}\}\).
\end{itemize}

For our plug-in, PANDA, and its ablations, we use the following additional hyperparameters on top of the backbone's hyperparameters:

\begin{itemize}
    \item GAT attention version \(\in \{\text{v1}, \text{v2}\}\); 
    \item Attention heads \(\in \{1, 2, \ldots, 8\}\); 
    \item Distinct attention weights for source and destination \(\in \{\text{True}, \text{False}\}\).
\end{itemize}

Calculating the attention scores, as described in Equation~\eqref{eq:alpha_att_ev}, requires knowledge of the previous layer's hyperedge representation.
Since all the considered datasets do not start with any initial hyperedge representation, we calculate the initial hyperedge representation using the backbone's vanilla scores.
For consistency, we use the same strategy in subsequent layers.

\subsection{Dataset Descriptions}\label{app:datasets-description}

The statistics of the 8 real-world datasets are provided in Table~\ref{tab:dataset_statistics}.
The datasets used for the experiments are:
\begin{itemize}
    \item  \texttt{Citeseer}, \texttt{Pubmed} 
Standard citation benchmarks in which vertices represent research papers and directed edges encode citation links. Node attributes are constructed from text using bag-of-words representations of the documents.

\item \texttt{email-Enron}
A corporate email communication network built from Enron’s message logs. Nodes correspond to email accounts and edges record sender interactions. As ground-truth labels are unavailable, we derive node classes via the Spinglass community detection method~\cite{Reichardt_2006}.

\item \texttt{Texas, Cornell}
WebKB datasets collected from university computer science departments. Each node is a webpage, hyperlinks are edges, and features are bag-of-words over page content. Pages are annotated into five categories: student, project, course, staff, and faculty.

\item \texttt{Telegram}
An interaction network extracted from Telegram, capturing exchanges among users who propagate political content.

\item \texttt{Squirrel}
The Squirrel dataset consists of articles from the English Wikipedia (December 2018). Nodes represent articles, and edges represent mutual links between them. Node features indicate the presence of specific nouns in the articles. Nodes are grouped into five categories based on the original regression targets.

\item \texttt{Roman-empire} The dataset is based on the \emph{Roman Empire} article from English Wikipedia, which was selected since it is one of the longest articles on Wikipedia and it follows the construction proposed by \cite{platonov2023a}. Each node in the graph corresponds to one (non-unique) word in the text.
\end{itemize}

\begin{table}[htb!]
\caption{Statistics of the datasets used in our experiments. Reported are the number of nodes, features, hyperedges, and classes, as well as the average hyperedge size ($\lvert e \rvert$), the average node degree ($\lvert v \rvert$), the clique-expansion (CE) homophily computed as in \cite{wang2022equivariant}, and the node and hyperedge homophily as calculated in \cite{heterophilyhypergraphs}.}
\label{tab:dataset_statistics}
\resizebox{\textwidth}{!}{%
\begin{tabular}{lrrrrrrrrr}
\toprule
\textbf{Dataset} & \textbf{\# Nodes} & \textbf{\# Features} & \textbf{\# Hyperedges} & \textbf{\# Classes} & \textbf{avg $|e|$} & \textbf{avg $|v|$} & \textbf{CE Hom.}  & \textbf{Node Hom.} & \textbf{Hyperedge Hom.} \\
\midrule
Roman-empire      & 22,662 & 300    & 22,662 & 18 & 2.73  & 2.73  & 0.2363 & 0.0582 & 0.0582\\
Squirrel    & 2,223  & 2,089  & 2,060  & 5  & 23.81 & 22.07 & 0.2448 & 0.2081 & 0.2152\\
Telegram    & 245    & 1      & 183    & 4  & 49.70 & 37.12 & 0.2854 & 0.3459 & 0.3784\\
email-Enron & 143    & --     & 139    & 7  & 19.58 & 19.03 & 0.3251 & 0.4163 & 0.4314\\
Cornell     & 183    & 1,703  & 96     & 5  & 4.07  & 2.14  & 0.4200 & 0.4674 & 0.4534\\
Texas       & 183    & 1,703  & 110   & 5  & 3.81 & 2.29 & 0.5049 & 0.5076 & 0.3172\\
Citeseer    & 3,312  & 3,703  & 1,951  & 6  & 3.35 & 1.98  & 0.7947 & 0.5771 & 0.6814\\
{Pubmed} & 19,717 & 500 & 7,963 & 3 & 4.35 & 1.76 & 0.8012 & 0.5499 & 0.7765\\
\bottomrule
\end{tabular}
}
\end{table}

\section{Computational and Parameter Comparison}
\label{sec:appendix-overheads}
We report the computational and parameter overhead of the plug-in here. We note that the plug-in introduces various parameters to the model, leading to a theoretical increase in computational and parameter overhead. However, since we run a hyperparameter search, the improved architecture can allow for faster and more parameter efficient models.

\begin{table}[tbp!]
\centering
\small
\caption{Parameter counts for PANDA and the vanilla models across different datasets and different backbone architectures.}
\label{tab:parameter-overhead}
\setlength{\tabcolsep}{3pt}
\renewcommand{\arraystretch}{0.95}
\resizebox{\textwidth}{!}{%
\begin{tabular}{l c c c c c c c c}
\toprule
\bfseries Model
& \texttt{Telegram} 
& \texttt{Citeseer}
& \texttt{Pubmed} 
& \texttt{Roman-empire} 
& \texttt{Cornell} 
& \texttt{Texas}
& \texttt{Squirrel}
& \texttt{email-Enron} \\
\midrule

GeDi (Vanilla) & 8.922 M & \textbf{30.864 M} & \textbf{0.578 M} & 5.383 M & \textbf{1.461} M & 4.015 M & 11.533 M & 26.884 M \\
GeDi + PANDA & \textbf{0.765 M} & 35.682 M & 6.714 M & \textbf{0.421 M} & 8.470 M & \textbf{0.638 M} & \textbf{1.902 M} & \textbf{5.880 M} \\

\midrule

HGNN (Vanilla) & \textbf{0.134 M} & \textbf{13.734 M} & 1.932 M & \textbf{1.761 M} & \textbf{10.861 M} & \textbf{3.168 M} & 3.860 M & \textbf{0.035 M} \\
HGNN + PANDA & 0.163 M & 15.720 M & \textbf{1.400 M} & 2.590 M & 12.358 M & 6.638 M & \textbf{3.366 M} & 0.139 M \\

\midrule

HyperGCN (Vanilla) & 0.695 M & 3.154 M & \textbf{0.059 M} & \textbf{0.046 M} & \textbf{0.051} M & \textbf{0.026 M} & 0.579 M & 0.396 M \\
HyperGCN + PANDA & \textbf{0.027 M} & \textbf{1.933 M} & 0.060 M & 0.138 M & 0.225 M & 0.060 M & \textbf{0.275 M} & \textbf{0.052 M} \\

\midrule

PhenomNN (Vanilla) & \textbf{0.001 M} & 1.662 M & 0.516 M & \textbf{0.135 M} & 5.464 M & \textbf{0.439 M} & 1.597 M & 0.407 M \\
PhenomNN + PANDA & 1.852 M & \textbf{1.534 M} & \textbf{0.066 M} & 0.150 M & \textbf{1.190 M} & 0.974 M & \textbf{0.345 M} & \textbf{0.230 M} \\

\bottomrule
\end{tabular}
}
\end{table}

\begin{table}[tbp!]
\centering
\small
\caption{Wall-clock inference times (mean $\pm$ standard deviation) for PANDA and the vanilla models across different datasets and different backbone architectures.}
\label{tab:computational-overhead}
\setlength{\tabcolsep}{3pt}
\renewcommand{\arraystretch}{0.95}
\resizebox{\textwidth}{!}{%
\begin{tabular}{l c c c c c c c c}
\toprule
\bfseries Model
& \texttt{Telegram} 
& \texttt{Citeseer}
& \texttt{Pubmed} 
& \texttt{Roman-empire} 
& \texttt{Cornell} 
& \texttt{Texas}
& \texttt{Squirrel}
& \texttt{email-Enron} \\
\midrule

GeDi (Vanilla) & 20.3 $\pm$ 0.1 ms & 50.9 $\pm$ 1.9 ms & 24.8 $\pm$ 0.2 ms & 203.8 $\pm$ 1.1 ms & 7.2 $\pm$ 0.0 ms & 5.5 $\pm$ 0.1 ms & 41.2 $\pm$ 0.9 ms & 12.8 $\pm$ 0.1 ms \\
GeDi + PANDA & 10.9 $\pm$ 0.0 ms & 122.6 $\pm$ 0.7 ms & 171.7 $\pm$ 1.0 ms & 115.9 $\pm$ 1.9 ms & 9.7 $\pm$ 0.0 ms & 8.2 $\pm$ 0.1 ms & 28.4 $\pm$ 0.1 ms & 33.2 $\pm$ 0.1 ms \\

\midrule

HGNN (Vanilla) & 1.2 $\pm$ 0.0 ms & 10.5 $\pm$ 0.0 ms & 16.8 $\pm$ 0.0 ms & 22.1 $\pm$ 0.0 ms & 1.9 $\pm$ 0.0 ms & 1.3 $\pm$ 0.0 ms & 5.9 $\pm$ 0.0 ms & 1.4 $\pm$ 0.0 ms \\
HGNN + PANDA & 5.7 $\pm$ 0.2 ms & 22.5 $\pm$ 1.4 ms & 42.1 $\pm$ 0.4 ms & 120.7 $\pm$ 0.5 ms & 5.5 $\pm$ 0.0 ms & 3.9 $\pm$ 0.0 ms & 23.6 $\pm$ 1.1 ms & 20.2 $\pm$ 0.0 ms \\

\midrule

HyperGCN (Vanilla) & 360.9 $\pm$ 69.2 ms & 74.3 $\pm$ 1.8 ms & 2.5 $\pm$ 0.0 ms & 1.3 $\pm$ 0.0 ms & 2.8 $\pm$ 0.2 ms & 2.7 $\pm$ 0.1 ms & 3.5 $\pm$ 0.0 ms & 5.9 $\pm$ 0.0 ms \\
HyperGCN + PANDA & 239.0 $\pm$ 5.2 ms & 50.8 $\pm$ 0.7 ms & 20.4 $\pm$ 0.0 ms & 31.1 $\pm$ 0.2 ms & 6.8 $\pm$ 0.0 ms & 3.6 $\pm$ 0.1 ms & 8.5 $\pm$ 0.0 ms & 17.1 $\pm$ 0.0 ms \\

\midrule

PhenomNN (Vanilla) & 25.1 $\pm$ 0.4 ms & 26.8 $\pm$ 0.1 ms & 178.3 $\pm$ 0.0 ms & 105.3 $\pm$ 7.5 ms & 15.5 $\pm$ 0.0 ms & 15.2 $\pm$ 0.1 ms & 132.4 $\pm$ 0.1 ms & 200.2 $\pm$ 17.9 ms \\
PhenomNN + PANDA & 103.4 $\pm$ 0.1 ms & 116.8 $\pm$ 0.2 ms & 75.4 $\pm$ 3.1 ms & 143.9 $\pm$ 5.4 ms & 99.7 $\pm$ 1.5 ms & 102.6 $\pm$ 0.6 ms & 94.9 $\pm$ 0.3 ms & 355.1 $\pm$ 7.9 ms \\

\bottomrule
\end{tabular}
}
\end{table}

Tables \ref{tab:parameter-overhead} and \ref{tab:computational-overhead} report the parameter and computational overheads
of PANDA and the vanilla models for the different backbone architectures across different datasets.
We see that despite PANDA itself adding various components, 50\% of the time, our hyperparameter search leads us to a
more parameter efficient model.
We also see that 20\% of the time, the PANDA variant has faster inference speed than the baseline model.

\section{Results on Additional Datasets}
\label{sec:appendix-additional-eval}
Due to space constraints, we only report the results of eight datasets in \Cref{tab:eval-merged}.
We provide results for additional datasets here.
We include results over \texttt{Wisconsin}~\citep{pei2020geomgcngeometricgraphconvolutional},
\texttt{email-EU}~\citep{benson-simplicial} and \texttt{Questions}~\citep{platonovcritical}.
Questions consists of 48,921 nodes and acts as a scalability test.

\Cref{tab:eval-additional} shows the results on these additional datasets. We see that for \texttt{Wisconsin}
and \texttt{email-EU}, HyperGCN benefits massively from PANDA, while the other three backbones see small but consistently positive gains or almost matching performance.
For \texttt{Questions} dataset, PANDA improves the performance across the board.

\begin{table}[tbp!]
\centering
\small
\caption{Node classification results reporting mean accuracies $\pm$ standard deviation for \texttt{Wisconsin} and \texttt{email-EU}. For \texttt{Questions}, we report the percentage AUC-ROC for each method, as this is a highly unbalanced dataset. We compare our PANDA plug-in and the vanilla model with different backbone architectures.}
\label{tab:eval-additional}
\setlength{\tabcolsep}{3pt}
\begin{tabular}{c l c c c}
\toprule
& \bfseries Variant
& \texttt{Wisconsin} 
& \texttt{email-EU}
& \texttt{Questions} \\
\midrule

\multirow{3}{*}{\rotatebox[origin=c]{90}{GeDi}}

& Vanilla & 87.45 $\pm$ 5.203 & 53.48 $\pm$ 4.53 & 78.29 $\pm$ 1.026 \\
& PANDA & 87.84 $\pm$ 4.454 & 52.87 $\pm$ 5.60 & 79.69 $\pm$ 1.21 \\

\cmidrule{2-5}

& $\Delta$(PANDA - Vanilla) & \textcolor{blue}{+0.39} & \textcolor{red}{-0.61} & \textcolor{blue}{+1.40} \\

\midrule\midrule

\multirow{3}{*}{\rotatebox[origin=c]{90}{HGNN}}

& Vanilla & 85.88 $\pm$ 5.173 & 52.88 $\pm$ 1.83 & 78.25 $\pm$ 1.122 \\
& PANDA & 85.88 $\pm$ 5.462 & 52.83 $\pm$ 6.80 & 79.16 $\pm$ 1.126 \\

\cmidrule{2-5}

& $\Delta$(PANDA - Vanilla) & \textcolor{yellow}{+0.00} & \textcolor{red}{-0.05} & \textcolor{blue}{+0.91} \\

\midrule\midrule

\addlinespace[0.5em]

\multirow{3}{*}[0.4em]{\rotatebox[origin=c]{90}{HyperGCN}}

& Vanilla & 60.59 $\pm$ 3.965 & 32.51 $\pm$ 2.99 & 72.97 $\pm$ 0.955 \\
& PANDA & 80.78 $\pm$ 6.188 & 43.32 $\pm$ 5.02 & 74.05 $\pm$ 0.659 \\

\cmidrule{2-5}

& $\Delta$(PANDA - Vanilla) & \textcolor{blue}{+20.19} & \textcolor{blue}{+10.81} & \textcolor{blue}{+1.08} \\[0.5em]

\midrule\midrule

\addlinespace[0.5em]

\multirow{3}{*}[0.4em]{\rotatebox[origin=c]{90}{PhenomNN}}

& Vanilla & 85.49 $\pm$ 5.203 & 42.55 $\pm$ 3.40 & 76.02 $\pm$ 0.938 \\
& PANDA & 85.88 $\pm$ 4.094 & 47.73 $\pm$ 6.83 & 76.15 $\pm$ 2.849 \\

\cmidrule{2-5}

& $\Delta$(PANDA - Vanilla) & \textcolor{blue}{+0.39} & \textcolor{blue}{+5.18} & \textcolor{blue}{+0.13} \\[0.5em]

\bottomrule
\end{tabular}
\end{table}

\begin{table}[tbp!]
\centering
\small
\caption{Parameter counts and wall-clock inference times (mean $\pm$ standard deviation) for PANDA and the vanilla models across different backbone architectures on \texttt{Questions} dataset.}
\label{tab:computational-parameter-questions}
\setlength{\tabcolsep}{3pt}
\begin{tabular}{l c c}
\toprule
\bfseries Model
& \bfseries Parameter Count 
& \bfseries Inference Time \\
\midrule

GeDi (Vanilla) & 0.511 M & 195.7 $\pm$ 0.8 ms \\
GeDi + PANDA & \textbf{0.427 M} & \textbf{72.8 $\pm$ 0.2 ms} \\

\midrule

HGNN (Vanilla) & 1.463 M & 53.4 $\pm$ 1.3 ms \\
HGNN + PANDA & \textbf{0.108 M} & \textbf{22.2 $\pm$ 0.0 ms} \\

\midrule

HyperGCN (Vanilla) & 0.213 M & \textbf{9.3 $\pm$ 0.0 ms} \\
HyperGCN + PANDA & \textbf{0.032 M} & 13.1 $\pm$ 0.0 ms \\

\midrule

PhenomNN (Vanilla) & 0.559 M & 1282.9 $\pm$ 83.7 ms \\
PhenomNN + PANDA & \textbf{0.041 M} & \textbf{431.4 $\pm$ 40.3 ms} \\

\bottomrule
\end{tabular}
\end{table}

\Cref{tab:computational-parameter-questions} shows the parameter and computational overhead for PANDA and the vanilla
models for different backbones on the \texttt{Questions} dataset. We notice that PANDA is consistently more parameter
efficient than the vanilla models across the backbones. We also find that PANDA leads to much faster inference for
all the backbones, except for HyperGCN, where it is 1.4 $\times$ slower. In fact, it is more than 2 $\times$ faster
than the vanilla model for all the other backbones.

\section{Embedding Stability and Degree Sensitivity}
\label{sec:appendix-stability}
The introduction of auxiliary paths in PANDA can lead to explosion or contraction of the node embedding norms.
Since this auxiliary path is not normalised, it means the model may be degree-sensitive, which in turn
may lead to disproportionate growth of the node embedding norms. We perform experiments to investigate whether PANDA
suffers from these phenomena.

\Cref{tab:stability} shows the average L2 norm growth of the node embeddings for GeDi and GeDi + PANDA
across layers for different datasets. We do notice any clear trend of increased instability due to the
PANDA plug-in: PANDA leads to comparable norms on Citeseer, Roman-empire and Cornell, larger norms on Telegram, Pubmed and email-Enron, and smaller ones on Texas and Squirrel.

\Cref{tab:degree-sensitivity} shows the average coefficient of variation for the L2 norm growth of the node embeddings for GeDi and GeDi + PANDA across layers for different datasets. We find that PANDA does not lead to disproportionate growth of node embedding norms. In fact, for almost half of the datasets, it exhibits a smaller coefficient of variation than the vanilla model.

\begin{table}[tbph!]
\centering
\small
\caption{Average L2 norm growth of the node embeddings across layers for different datasets. A value of 1 means the L2 norm does not change across layers, <1 means the embeddings are contracting, while >1 means that they are expanding.}
\label{tab:stability}
\setlength{\tabcolsep}{3pt}
\renewcommand{\arraystretch}{0.95}
\resizebox{\textwidth}{!}{%
\begin{tabular}{l c c c c c c c c}
\toprule
\bfseries Model
& \texttt{Telegram} 
& \texttt{Citeseer}
& \texttt{Pubmed} 
& \texttt{Roman-empire} 
& \texttt{Cornell} 
& \texttt{Texas}
& \texttt{Squirrel}
& \texttt{email-Enron} \\
\midrule

GeDi (Vanilla) & 29.6782 & 1.0956 & 0.1784 & 1.2821 & 1.1428 & 1.0881 & 11.038 & 1.1938 \\
GeDi + PANDA & 61.9474 & 0.8936 & 0.9324 & 1.4452 & 1.5472 & 0.4988 & 1.8974 & 6.2111 \\

\bottomrule
\end{tabular}
}
\end{table}

\begin{table}[tbph!]
\centering
\small
\caption{Average coefficient of variation for the L2 norm growth of the node embeddings across layers for different datasets. A value of 0 means that each node's L2 norm grows equally, while a high value means that the increase is disproportionate.}
\label{tab:degree-sensitivity}
\setlength{\tabcolsep}{3pt}
\renewcommand{\arraystretch}{0.95}
\resizebox{\textwidth}{!}{%
\begin{tabular}{l c c c c c c c c}
\toprule
\bfseries Model
& \texttt{Telegram} 
& \texttt{Citeseer}
& \texttt{Pubmed} 
& \texttt{Roman-empire} 
& \texttt{Cornell} 
& \texttt{Texas}
& \texttt{Squirrel}
& \texttt{email-Enron} \\
\midrule

GeDi (Vanilla) & 1.711 & 1.848 & 1.430 & 1.124 & 1.143 & 1.462 & 1.174 & 1.178 \\
GeDi + PANDA & 1.226 & 1.956 & 1.663 & 0.990 & 1.236 & 2.324 & 1.761 & 0.961 \\

\bottomrule
\end{tabular}
}
\end{table}

\section{Extended analysis}
\label{sec:analysis-app}

We analyse the learned GeDi-PANDA routing variables on Citeseer, Cornell,
Roman-empire, Telegram, and Texas. For each dataset, we reload the ten saved
checkpoints corresponding to the ten data splits and evaluate all quantities in
inference mode. Reported accuracy values are means \(\pm\) standard deviations
over the ten checkpoints. Routing statistics are also reported as means
\(\pm\) standard deviations over the ten checkpoints. All standard
deviations use the unbiased (sample) estimator. For the multi-layer
Roman-empire model, routing statistics are first computed at each PANDA layer
and then summarised over all layer--checkpoint instances, so the reported
standard deviations mix layer-to-layer variation with checkpoint-to-checkpoint
variation and are not directly comparable to the single-layer datasets; when
this pooling hides substantial layer structure, we state the layerwise behaviour
explicitly.

Table~\ref{tab:analysis-symbols} collects the symbols used throughout this
section. Routing variables describe how each PANDA layer combines learned
attention, the structural prior, and the auxiliary path; auxiliary-path
quantities are computed at inference time on the test split, with both
per-incidence auxiliary gates forced to zero in the aux-off baseline.

\begin{table}[t]
    \centering
    \caption{Symbols used in the ``Extended Analysis'' section.}
    \label{tab:analysis-symbols}
    \footnotesize
    \setlength{\tabcolsep}{4pt}
    \renewcommand{\arraystretch}{1.15}
    \begin{tabular}{@{}lp{0.78\linewidth}@{}}
        \toprule
        \multicolumn{2}{@{}l}{\emph{Routing variables}}                                            \\
        \midrule
        \(\alpha^{\mathrm{main}},\,\alpha^{\mathrm{aux}}\)
         & Main and auxiliary incidence coefficients ($\hat{\alpha}$ in \(\VE\),
        $\check{\alpha}$ in \(\EV\)).                                                              \\
        \(\alpha^{\mathrm{gat}}\)
         & Learned GAT component of the main coefficient.                                          \\
        \(\beta^{\mathrm{prior}}\)
         & Per-incidence prior gate; large value \(\Rightarrow\)
        structural prior dominates main path.                                                      \\
        \(g\)
         & Post-sigmoid auxiliary gate.                                                            \\
        \(\eta_{\mathrm{unif}}^{\mathrm{gat}}\)
         & Fraction of non-singleton receiver groups with effectively
        uniform \(\alpha^{\mathrm{gat}}\).                                                         \\
        \(\rho_{s,R},\,\pi_{s,R}\)
         & Aux-to-main norm ratio and aux share of total norm, in
        stage \(s\in\{\VE,\EV\}\), role \(R\in\{H,T\}\).                                           \\
        \midrule
        \multicolumn{2}{@{}l}{\emph{Auxiliary-path quantities}}                                    \\
        \midrule
        \(\mathrm{Acc}^{\mathrm{full}},\,\mathrm{Acc}^{\mathrm{aux\text{-}off}}\)
         & Test accuracy of the full model and with both per-incidence
        auxiliary gates forced to zero.                                                            \\
        \(\Delta\mathrm{Acc.}\)
         & \(\mathrm{Acc}^{\mathrm{aux\text{-}off}}-\mathrm{Acc}^{\mathrm{full}}\)
        in percentage points (negative means closing the auxiliary
        path hurts).                                                                               \\
        \(\mathcal H_{\mathrm{CE}}\)
         & Dataset-level clique-expansion homophily~\citep{wang2023hypergraph}.                    \\
        \(y_u^{\mathrm{nbr}}\)
         & Majority class among \(u\)'s clique-expansion neighbours.                                \\
        \(q_{\mathrm{mis}}\)
         & Fraction of CE-valid test nodes
        \(u\in\mathcal T_{\mathrm{CE}}\) with
        \(y_u^{\mathrm{nbr}}\neq y_u\) (\emph{misleading neighbour}).                               \\
        \(q_{\mathrm{help}},\,q_{\mathrm{hurt}}\)
         & Fraction of test nodes the auxiliary path flips onto / off
        the true class.                                                                            \\
        \(q_{\mathrm{nbr\text{-}rescue}}\)
         & Fraction of all test nodes rescued from the misleading
        neighbour-majority class to the true class (aux-off predicts
        \(y_u^{\mathrm{nbr}}\), full predicts \(y_u\)); on the same
        \(|\mathcal T|\) scale as \(q_{\mathrm{help}}\).                                           \\
        \(\overline{\Delta}^{\mathrm{nbr}}_{\neq}\)
         & Mean true-vs-neighbour-majority logit-margin shift on
        misleading-neighbour nodes; positive means the auxiliary path pushes towards the true class. \\
        \(\overline{\Delta}^{\mathrm{max}}\)
         & Mean true-vs-best-other-class logit-margin shift over all
        test nodes.                                                                                \\
        \bottomrule
    \end{tabular}
\end{table}

\subsection{Routing statistics and gate interventions}

The model has two routing stages. In the node-to-hyperedge stage, denoted
\(\VE\), hyperedges receive messages from incident nodes, with main and
auxiliary coefficients
\(\textcolor{cyan}{\hat\alpha^{\mathrm{main}}_{e,v}}\) and
\(\textcolor{red}{\hat\alpha^{\mathrm{aux}}_{e,v}}\). In the hyperedge-to-node stage, denoted
\(\EV\), nodes receive messages from incident hyperedges, with coefficients
\(\textcolor{orange}{\check\alpha^{\mathrm{main}}_{u,e}}\) and
\(\textcolor{blue}{\check\alpha^{\mathrm{aux}}_{u,e}}\). In both stages, the main coefficient
interpolates between learned attention and the structural prior, while the
auxiliary coefficient is bounded by the remaining local capacity,
\(1-\alpha^{\mathrm{main}}\), and modulated by a learned sigmoid gate.

Table~\ref{tab:analysis-summary} first reports basic hypergraph statistics:
the number of nodes \(|\mathcal V|\), the number of hyperedges \(|\mathcal E|\),
the average hyperedge size \(\overline{|e|}\), and the average node degree
\(\overline{|\Gamma(v)|}\). Then, it
summarises six routing quantities. We denote by
\(\eta_{\mathrm{unif}}^{\mathrm{gat}}\) the fraction of hyperedges (in the \(\VE\) stage) or nodes (in the \(\EV\) stage) for
which the learned attention component is effectively uniform over the same set
used in the corresponding softmax. To avoid trivially-uniform singleton receiver
groups, define the multi-element eligible sets
\(\mathcal{E}^{\geq 2}=\{e\in\mathcal{E}:|H(e)\cup T(e)|\geq 2\}\) and
\(\mathcal{V}^{\geq 2}=\{u\in\mathcal{V}:|\Gamma_H(u)\cup\Gamma_T(u)|\geq 2\}\).
With tolerance \(\varepsilon=10^{-6}\), we compute
\begin{equation}
  \begin{aligned}
    \eta_{\mathrm{unif}}^{\mathrm{gat},\VE}
     & =
    \frac{1}{|\mathcal{E}^{\geq 2}|}
    \sum_{e\in\mathcal{E}^{\geq 2}}
    \mathbb{I}\!\left[
    \max_{v\in H(e)\cup T(e)}
    \left|\hat\alpha_{e,v}^{\mathrm{gat}}
    -\frac{1}{|H(e)\cup T(e)|}\right|
    < \varepsilon
    \right],
    \\
    \eta_{\mathrm{unif}}^{\mathrm{gat},\EV}
     & =
    \frac{1}{|\mathcal{V}^{\geq 2}|}
    \sum_{u\in\mathcal{V}^{\geq 2}}
    \mathbb{I}\!\left[
    \max_{e\in \Gamma_H(u)\cup \Gamma_T(u)}
    \left|\check\alpha_{u,e}^{\mathrm{gat}}
    -\frac{1}{|\Gamma_H(u)\cup \Gamma_T(u)|}\right|
    < \varepsilon
    \right].
  \end{aligned}
  \label{eq:eta-uniform-gat}
\end{equation}
For multi-head models, this same check is applied independently to each
attention head and then averaged over heads; we suppress the head index in the
notation. Singleton receiver groups (\(|H(e)\cup T(e)|=1\) in the \(\VE\)
stage, or \(|\Gamma_H(u)\cup \Gamma_T(u)|=1\) in the \(\EV\) stage) are
trivially uniform and are excluded from both the numerator and denominator of
\(\eta_{\mathrm{unif}}^{\mathrm{gat}}\) so that the reported fraction reflects
non-trivial uniformity only.\footnote{On the \(\VE\) side, singleton groups
  are size-1 hyperedges (self-loops, either present in the source data or
  inserted during preprocessing for otherwise-isolated nodes). On the \(\EV\)
  side, they are predominantly \emph{leaf nodes} -- nodes incident to exactly
  one multi-node hyperedge -- rather than isolated nodes; the leaf fraction is
  substantial on most datasets (Citeseer \(53.8\%\), Cornell \(48.1\%\), Texas
  \(42.1\%\), Telegram \(21.6\%\), Roman-empire \(0\%\)). The post-exclusion
  \(\eta^{\mathrm{gat},\EV}_{\mathrm{unif}}\) is therefore the uniformity rate
  among nodes for which attention has more than one source to choose
  between.}
The same \(\eta_{\mathrm{unif}}^{\mathrm{gat}}\) value is reported for head and
tail rows within a stage because learned GAT attention is normalised over the
same set, i.e., \(H(e)\cup T(e)\) in the \(\VE\) stage and
\(\Gamma_H(u)\cup\Gamma_T(u)\) in the \(\EV\) stage. The remaining columns are role-split at analysis time by masking incidences
according to their role: the \(H\) row in the \(\VE\) stage averages over
\(\{(e,v):v\in H(e)\}\), the \(T\) row over \(\{(e,v):v\in T(e)\}\), analogously in the \(\EV\) stage over \(\{(u,e):e\in \Gamma_H(u)\}\) and \(\{(u,e):e\in \Gamma_T(u)\}\). The model itself still computes
\(\alpha^{\mathrm{gat}}\), \(\beta^{\mathrm{prior}}\), the gates
\(\hat g, \check g\), and \(\alpha^{\mathrm{aux}}\) on the full incidence
support; the table only separates their empirical values by head and tail role.

\begin{table*}[t]
  \centering
  \caption{Analysis summary of learned incidence routing and the auxiliary
    path. The accuracy column reports mean \(\pm\) standard deviation test
    accuracy in percent over ten saved checkpoints. Dataset statistics are
    computed on the loaded hypergraphs. Routing values are means \(\pm\)
    standard deviations over the same checkpoints, and over layer--checkpoint
    instances for Roman-empire. The fraction of uniform attention is computed
    over non-singleton receiver groups, namely \(H(e)\cup T(e)\) with
    \(|H(e)\cup T(e)|\geq 2\) in the \(\VE\) stage and
    \(\Gamma_H(u)\cup\Gamma_T(u)\) with \(|\Gamma_H(u)\cup\Gamma_T(u)|\geq 2\)
    in the \(\EV\) stage.}
  \label{tab:analysis-summary}
  \scriptsize
  \setlength{\tabcolsep}{1.8pt}
  \renewcommand{\arraystretch}{1.12}
  \newcommand{\mstd}[2]{\ensuremath{#1{\scriptscriptstyle\pm}#2}}

  \resizebox{\textwidth}{!}{%
    \begin{tabular}{@{}lccccccccccccc@{}}
      \toprule
      Dataset
       & \makecell{Test acc.                                                                                                                                                                                                                                                                                                         \\mean$\pm$std}
       & \makecell{\(|\mathcal V|\)}
       & \makecell{\(|\mathcal E|\)}
       & \makecell{\(\overline{|e|}\)}
       & \makecell{\(\overline{|\Gamma(v)|}\)}
       & Stage
       & Role
       & \makecell{$\eta_{\mathrm{unif}}^{\mathrm{gat}}$}
       & \makecell{$\beta^{\mathrm{prior}}$}
       & \makecell{$g$}
       & \makecell{$\alpha^{\mathrm{aux}}$}
       & \(\rho_{s,R}\)
       & \(\pi_{s,R}\)                                                                                                                                                                                                                                                                                                               \\
      \midrule

      \multirow{4}{*}{Citeseer}
       & \multirow{4}{*}{\(75.6\pm1.2\)}
       & \multirow{4}{*}{3312}
       & \multirow{4}{*}{1951}
       & \multirow{4}{*}{3.35}
       & \multirow{4}{*}{1.98}
       & \multirow{2}{*}{\(\VE\)}
       & H                                                & \multirow{2}{*}{\mstd{0.000}{0.000}} & \mstd{0.863}{0.045} & \mstd{0.378}{0.041} & \mstd{0.243}{0.027} & \mstd{0.320}{0.047}      & \mstd{0.241}{0.027}                                                                                                                  \\
       &                                                  &                                      &                     &                     &                     &                          & T                   &  & \mstd{0.862}{0.046} & \mstd{0.395}{0.040} & \mstd{0.253}{0.027} & \mstd{0.335}{0.047} & \mstd{0.250}{0.027} \\
       &                                                  &                                      &                     &                     &                     & \multirow{2}{*}{\(\EV\)}
       & H                                                & \multirow{2}{*}{\mstd{0.005}{0.000}} & \mstd{0.930}{0.040} & \mstd{0.390}{0.052} & \mstd{0.242}{0.033} & \mstd{0.269}{0.041}      & \mstd{0.211}{0.025}                                                                                                                  \\
       &                                                  &                                      &                     &                     &                     &                          & T                   &  & \mstd{0.844}{0.042} & \mstd{0.399}{0.050} & \mstd{0.247}{0.032} & \mstd{0.293}{0.042} & \mstd{0.226}{0.025} \\
      \addlinespace[2pt]

      \multirow{4}{*}{Cornell}
       & \multirow{4}{*}{\(74.9\pm4.8\)}
       & \multirow{4}{*}{183}
       & \multirow{4}{*}{96}
       & \multirow{4}{*}{4.07}
       & \multirow{4}{*}{2.14}
       & \multirow{2}{*}{\(\VE\)}
       & H                                                & \multirow{2}{*}{\mstd{0.429}{0.001}} & \mstd{0.368}{0.100} & \mstd{0.181}{0.079} & \mstd{0.140}{0.062} & \mstd{0.294}{0.131}      & \mstd{0.220}{0.078}                                                                                                                  \\
       &                                                  &                                      &                     &                     &                     &                          & T                   &  & \mstd{0.379}{0.093} & \mstd{0.207}{0.090} & \mstd{0.117}{0.051} & \mstd{0.124}{0.056} & \mstd{0.108}{0.044} \\
       &                                                  &                                      &                     &                     &                     & \multirow{2}{*}{\(\EV\)}
       & H                                                & \multirow{2}{*}{\mstd{0.018}{0.057}} & \mstd{0.832}{0.055} & \mstd{0.292}{0.083} & \mstd{0.215}{0.066} & \mstd{0.204}{0.089}      & \mstd{0.166}{0.062}                                                                                                                  \\
       &                                                  &                                      &                     &                     &                     &                          & T                   &  & \mstd{0.732}{0.096} & \mstd{0.368}{0.083} & \mstd{0.190}{0.044} & \mstd{0.183}{0.050} & \mstd{0.154}{0.036} \\
      \addlinespace[2pt]

      \multirow{4}{*}{Roman-empire}
       & \multirow{4}{*}{\(89.2\pm0.5\)}
       & \multirow{4}{*}{22662}
       & \multirow{4}{*}{22662}
       & \multirow{4}{*}{2.73}
       & \multirow{4}{*}{2.73}
       & \multirow{2}{*}{\(\VE\)}
       & H                                                & \multirow{2}{*}{\mstd{0.184}{0.270}} & \mstd{0.831}{0.079} & \mstd{0.497}{0.140} & \mstd{0.329}{0.093} & \mstd{0.959}{0.262}      & \mstd{0.480}{0.074}                                                                                                                  \\
       &                                                  &                                      &                     &                     &                     &                          & T                   &  & \mstd{0.828}{0.078} & \mstd{0.490}{0.129} & \mstd{0.306}{0.077} & \mstd{0.805}{0.219} & \mstd{0.438}{0.071} \\
       &                                                  &                                      &                     &                     &                     & \multirow{2}{*}{\(\EV\)}
       & H                                                & \multirow{2}{*}{\mstd{0.001}{0.004}} & \mstd{0.908}{0.066} & \mstd{0.370}{0.132} & \mstd{0.247}{0.085} & \mstd{0.764}{0.250}      & \mstd{0.423}{0.076}                                                                                                                  \\
       &                                                  &                                      &                     &                     &                     &                          & T                   &  & \mstd{0.874}{0.075} & \mstd{0.413}{0.134} & \mstd{0.268}{0.082} & \mstd{0.830}{0.328} & \mstd{0.440}{0.082} \\
      \addlinespace[2pt]

      \multirow{4}{*}{Telegram}
       & \multirow{4}{*}{\(89.6\pm5.2\)}
       & \multirow{4}{*}{245}
       & \multirow{4}{*}{183}
       & \multirow{4}{*}{49.70}
       & \multirow{4}{*}{37.12}
       & \multirow{2}{*}{\(\VE\)}
       & H                                                & \multirow{2}{*}{\mstd{1.000}{0.000}} & \mstd{0.500}{0.000} & \mstd{0.411}{0.031} & \mstd{0.404}{0.031} & \mstd{27.578}{4.087}     & \mstd{0.964}{0.005}                                                                                                                  \\
       &                                                  &                                      &                     &                     &                     &                          & T                   &  & \mstd{0.500}{0.000} & \mstd{0.410}{0.031} & \mstd{0.382}{0.029} & \mstd{6.423}{0.934} & \mstd{0.864}{0.016} \\
       &                                                  &                                      &                     &                     &                     & \multirow{2}{*}{\(\EV\)}
       & H                                                & \multirow{2}{*}{\mstd{0.000}{0.000}} & \mstd{0.551}{0.292} & \mstd{0.123}{0.025} & \mstd{0.120}{0.025} & \mstd{3.153}{0.507}      & \mstd{0.756}{0.029}                                                                                                                  \\
       &                                                  &                                      &                     &                     &                     &                          & T                   &  & \mstd{0.553}{0.221} & \mstd{0.228}{0.028} & \mstd{0.200}{0.028} & \mstd{0.533}{0.441} & \mstd{0.306}{0.163} \\
      \addlinespace[2pt]

      \multirow{4}{*}{Texas}
       & \multirow{4}{*}{\(83.0\pm5.1\)}
       & \multirow{4}{*}{183}
       & \multirow{4}{*}{110}
       & \multirow{4}{*}{3.81}
       & \multirow{4}{*}{2.29}
       & \multirow{2}{*}{\(\VE\)}
       & H                                                & \multirow{2}{*}{\mstd{0.414}{0.035}} & \mstd{0.490}{0.034} & \mstd{0.272}{0.067} & \mstd{0.211}{0.053} & \mstd{0.393}{0.205}      & \mstd{0.270}{0.090}                                                                                                                  \\
       &                                                  &                                      &                     &                     &                     &                          & T                   &  & \mstd{0.490}{0.031} & \mstd{0.300}{0.078} & \mstd{0.166}{0.043} & \mstd{0.136}{0.038} & \mstd{0.119}{0.030} \\
       &                                                  &                                      &                     &                     &                     & \multirow{2}{*}{\(\EV\)}
       & H                                                & \multirow{2}{*}{\mstd{0.001}{0.002}} & \mstd{0.785}{0.106} & \mstd{0.304}{0.106} & \mstd{0.231}{0.086} & \mstd{0.226}{0.141}      & \mstd{0.175}{0.093}                                                                                                                  \\
       &                                                  &                                      &                     &                     &                     &                          & T                   &  & \mstd{0.691}{0.041} & \mstd{0.399}{0.081} & \mstd{0.201}{0.042} & \mstd{0.173}{0.065} & \mstd{0.145}{0.047} \\
      \bottomrule
    \end{tabular}%
  }
  \vspace{2pt}
  \begin{minipage}{0.98\textwidth}
    \footnotesize
    \emph{Notes.}
    In the \(\VE\) stage, the routing variables are
    \(\hat\alpha^{\mathrm{gat}}_{e,v}\),
    \(1-\hat\beta^{(\ell)}_{e,v}\),
    \(\textcolor{red}{\hat g_{e,v}}\), and
    \(\textcolor{red}{\hat\alpha^{\mathrm{aux}}_{e,v}}\).
    In the \(\EV\) stage, they are
    \(\check\alpha^{\mathrm{gat}}_{u,e}\),
    \(1-\check\beta^{(\ell)}_{u,e}\),
    \(\textcolor{blue}{\check g_{u,e}}\), and
    \(\textcolor{blue}{\check\alpha^{\mathrm{aux}}_{u,e}}\).
    \(\rho_{s,R}\) and \(\pi_{s,R}\) are defined in Eq.~\eqref{eq:rho-pi}, for
    stage $s \in \{\VE, \EV\}$ and role $R \in \{H,T\}$.
  \end{minipage}
\end{table*}

On these role-restricted incidence sets, the columns
\(\beta^{\mathrm{prior}}\), \(g\), and
\(\alpha^{\mathrm{aux}}\) report, respectively, the mean prior-gate mass, the
mean post-sigmoid auxiliary gate, and the mean auxiliary coefficient,
averaged over heads. The auxiliary-gate column \(g\) refers to
\(\textcolor{red}{\hat g_{e,v}}\) in the \(\VE\) stage and to
\(\textcolor{blue}{\check g_{u,e}}\) in the \(\EV\) stage, as specified in
the table notes. The reported \(\beta^{\mathrm{prior}}\) is the empirical
incidence-level mean of the per-incidence prior mass
\(1-\hat\beta^{(\ell)}_{e,v}\) (resp.\ \(1-\check\beta^{(\ell)}_{u,e}\)) from
the convex combination of Eq.~\eqref{eq:main-coefficients}, averaged over the
H- or T-restricted incidence set, over heads, and over checkpoints; a high
\(\beta^{\mathrm{prior}}\) corresponds to the prior dominating the main path.
Because the per-incidence gates \(\hat\beta^{(\ell)}_{e,v}\),
\(\check\beta^{(\ell)}_{u,e}\) take different values across incidences, the
H and T rows can therefore differ.
The table then reports the mean \(\pm\) standard deviation of these
checkpoint-level quantities. Accuracy is reported with the same mean \(\pm\)
standard deviation convention.

The last two columns compare the magnitudes of the main and auxiliary messages.
For this purpose, we define the incidence-level messages
\begin{equation}
  \begin{aligned}
    \textcolor{cyan}{\hat m_{e,v}^{\mathrm{main}}}
     & =
    \textcolor{cyan}{\hat\alpha_{e,v}^{\mathrm{main}}}\,\mathbf{x}_v^{\mathrm{main}},
     &
    \textcolor{red}{\hat m_{e,v}^{\mathrm{aux}}}
     & =
    \textcolor{red}{\hat\alpha_{e,v}^{\mathrm{aux}}}\,\mathbf{x}_v^{\mathrm{aux}},
    \\
    \textcolor{orange}{\check m_{u,e}^{\mathrm{main}}}
     & =
    \textcolor{orange}{\check\alpha_{u,e}^{\mathrm{main}}}\,\mathbf{h}_e^{\mathrm{main}},
     &
    \textcolor{blue}{\check m_{u,e}^{\mathrm{aux}}}
     & =
    \textcolor{blue}{\check\alpha_{u,e}^{\mathrm{aux}}}\,\mathbf{h}_e^{\mathrm{aux}} .
  \end{aligned}
  \label{eq:analysis-messages}
\end{equation}
For the head side, let
\begin{equation}
  \begin{aligned}
    \textcolor{cyan}{M_{\VE,H}^{\mathrm{main}}}
     & =
    \sum_{e\in\mathcal{E}}\sum_{v\in H(e)}
    \|\textcolor{cyan}{\hat m_{e,v}^{\mathrm{main}}}\|_2,
     &
    \textcolor{red}{M_{\VE,H}^{\mathrm{aux}}}
     & =
    \sum_{e\in\mathcal{E}}\sum_{v\in H(e)}
    \|\textcolor{red}{\hat m_{e,v}^{\mathrm{aux}}}\|_2,
    \\
    \textcolor{orange}{M_{\EV,H}^{\mathrm{main}}}
     & =
    \sum_{u\in\mathcal{V}}\sum_{e\in\Gamma_H(u)}
    \|\textcolor{orange}{\check m_{u,e}^{\mathrm{main}}}\|_2,
     &
    \textcolor{blue}{M_{\EV,H}^{\mathrm{aux}}}
     & =
    \sum_{u\in\mathcal{V}}\sum_{e\in\Gamma_H(u)}
    \|\textcolor{blue}{\check m_{u,e}^{\mathrm{aux}}}\|_2 .
  \end{aligned}
  \label{eq:message-magnitudes}
\end{equation}
The corresponding tail-side quantities are obtained by replacing \(H(e)\) with
\(T(e)\) in the \(\VE\) sums and \(\Gamma_H(u)\) with \(\Gamma_T(u)\) in the
\(\EV\) sums. From these aggregate magnitudes, we report
\begin{equation}
  \rho_{s,R}
  =
  \frac{M_{s,R}^{\mathrm{aux}}}{M_{s,R}^{\mathrm{main}}},
  \qquad
  \pi_{s,R}
  =
  \frac{M_{s,R}^{\mathrm{aux}}}
  {M_{s,R}^{\mathrm{main}} + M_{s,R}^{\mathrm{aux}}} ,
  \label{eq:rho-pi}
\end{equation}
where \(s\in\{\VE,\EV\}\) and \(R\in\{H,T\}\). Thus, \(\rho>1\) means that
the auxiliary path has larger total norm than the main path, while
\(\pi=0\) means that the auxiliary path is inactive and \(\pi=1\) means that
almost all message magnitude comes from the auxiliary path. These are magnitude
statistics, not causal tests: the auxiliary representation includes learned
linear maps, so the functional role of a large norm must be checked by
intervention.

Table~\ref{tab:analysis-summary} shows that PANDA realises different routing
regimes across datasets. Telegram is the clearest high-auxiliary-magnitude
regime. Its learned \(\VE\) attention component is uniform in all eligible
non-singleton receiver groups and the empirical prior mass is fixed at
\(\beta^{\mathrm{prior}}=0.500\pm0.000\). Therefore, in the \(\VE\) stage, the
learned GAT component does not provide non-uniform routing; any
non-uniformity in the main path must therefore come from the structural
prior rather than from learned GAT attention, while the auxiliary path
carries most of the message magnitude. Specifically,
\(\pi_{\VE,H}=0.964\pm0.005\) on head-side incidences and
\(\pi_{\VE,T}=0.864\pm0.016\) on tail-side incidences. Equivalently,
\(\rho_{\VE,H}=27.578\pm4.087\) and \(\rho_{\VE,T}=6.423\pm0.934\), so the
auxiliary norm is much larger than the main-path norm in the \(\VE\) stage.
Because these are norm ratios, the interventions below are needed to establish
that this high-magnitude auxiliary path is also predictive.

We note that the $\beta^{prior} = 0.500 \pm 0.000$ pinning is specific to Telegram's $V \rightarrow E$ stage. This is likely due to Telegram's degenerate node features, rather than being a general failure of the interpolation mechanism. Indeed, on Telegram each node carries a single scalar feature $x_v \sim \mathcal{N}(0,1)$ which is drawn independently of labels and topology. This feature is clearly poorly informative in the $V \rightarrow E$ layer, where all the nodes belonging to a hyperedge are scored based on embeddings calculated after a linear layer from these uninformative 1-dimensional init features (and also against the same, shared hyperedge representation $h_e$). Thus, the value of $\beta^{prior}$ calculated in Eqs. (12) and (13) at the $V \rightarrow E$ stage is entirely uninformative. It follows that the data-dependent component of the CE part of the loss likely provides no informative gradient signal. Thus, the L2 penalty term of the loss becomes dominant, leading to a choice of weights $W_\beta$ close to zero. This behaviour is observed in all ten independent runs, with the trained $W_\beta$, $u_\beta$ are both extremely small (of the order of $10^{-3}$ in both cases, in Frobenius norm). As a result, the sigmoid gate in both equations collapses due to the learned scores $W$ being close to zero for every node-hyperedge incidence, yielding $\beta^{prior} = 0.500 \pm 0.000$.

In contrast, the $E \rightarrow V$ stage of the two-stage message-passing protocol scores the nodes against aggregated hyperedge representations. There, the scoring parameters retain norms that are one to two orders of magnitude larger and produce non-uniform attention, confirming that the observed degeneracy is caused by the input embedding of the $V \rightarrow E$ stage rather than by the attention mechanism itself. \Cref{tab:analysis-norms-telegram} supports this hypothesis.

\begin{table}
  \centering
  \caption{Norms of weight matrices and bias vectors in each stage of message passing for the single layer GeDi + PANDA model for Telegram dataset.}
  \label{tab:analysis-norms-telegram}
  \begin{tabular}{c c c}
    \toprule
    Stage &	norm($W_\beta$) &	norm($u_\beta$) \\
    \midrule
    $V \rightarrow E$	& 0.00272 $\pm$ 0.00436	& 0.07674 $\pm$ 0.14079 \\
    $E \rightarrow V$	& 0.54175 $\pm$ 0.32409	& 0.41068 $\pm$ 0.23789 \\
    \bottomrule
  \end{tabular}
\end{table}

\begin{table*}[t]
  \centering
  \caption{Auxiliary-gate shut-off interventions. The intervention column
    specifies which post-sigmoid auxiliary gate value is forced to zero. Values
    are mean \(\pm\) standard deviation test-accuracy changes in percentage
    points over ten saved checkpoints; baseline accuracy is reported with the
    same convention. Negative values mean that the intervention decreases
    accuracy. Note that this measures the dependence of the model on the auxiliary path,
    not the marginal contribution of this path.}
  \label{tab:analysis-intervention}
  \small
  \setlength{\tabcolsep}{4pt}
  \begin{tabular*}{\textwidth}{@{\extracolsep{\fill}}lllr@{}}
    \toprule
    Dataset
    & \makecell{Test acc.\\mean$\pm$std}
    & Intervention
    & \makecell{$\Delta$Acc. after\\gate shut-off} \\
    \midrule
    \multirow{3}{*}{Citeseer}
    & \multirow{3}{*}{\(75.6\pm1.2\)}
    & \(\textcolor{red}{\hat g^{(\ell)}_{e,v}}=0\) & \(-0.5\pm0.5\) \\
    &       & \(\textcolor{blue}{\check g^{(\ell)}_{u,e}}=0\) & \(-0.5\pm0.3\) \\
    &       & both gates \(=0\)                                            & \(-1.0\pm0.6\) \\
    \midrule
    \multirow{3}{*}{Cornell}
    & \multirow{3}{*}{\(74.9\pm4.8\)}
    & \(\textcolor{red}{\hat g^{(\ell)}_{e,v}}=0\) & \(-4.1\pm7.9\) \\
    &       & \(\textcolor{blue}{\check g^{(\ell)}_{u,e}}=0\) & \(-0.5\pm1.1\) \\
    &       & both gates \(=0\)                                            & \(-4.6\pm9.7\) \\
    \midrule
    \multirow{3}{*}{Roman-empire}
    & \multirow{3}{*}{\(89.2\pm0.5\)}
    & \(\textcolor{red}{\hat g^{(\ell)}_{e,v}}=0\) & \(-19.4\pm3.7\) \\
    &       & \(\textcolor{blue}{\check g^{(\ell)}_{u,e}}=0\) & \(-21.0\pm2.6\) \\
    &       & both gates \(=0\)                                            & \(-27.7\pm3.3\) \\
    \midrule
    \multirow{3}{*}{Telegram}
    & \multirow{3}{*}{\(89.6\pm5.2\)}
    & \(\textcolor{red}{\hat g^{(\ell)}_{e,v}}=0\) & \(-53.8\pm12.7\) \\
    &       & \(\textcolor{blue}{\check g^{(\ell)}_{u,e}}=0\) & \(-41.0\pm9.2\) \\
    &       & both gates \(=0\)                                            & \(-54.0\pm5.2\) \\
    \midrule
    \multirow{3}{*}{Texas}
    & \multirow{3}{*}{\(83.0\pm5.1\)}
    & \(\textcolor{red}{\hat g^{(\ell)}_{e,v}}=0\) & \(-1.4\pm5.6\) \\
    &       & \(\textcolor{blue}{\check g^{(\ell)}_{u,e}}=0\) & \(+0.5\pm2.1\) \\
    &       & both gates \(=0\)                                            & \(+0.0\pm4.0\) \\
    \bottomrule
  \end{tabular*}
\end{table*}

Roman-empire also uses the auxiliary path substantially, but with a different
profile. Its \(\VE\) learned-attention uniformity is
\(0.184\pm0.270\) after pooling over the four PANDA layers, which reflects the layer structure: the
layerwise means are \(0.548\), \(0.088\), \(0.090\), and \(0.010\) from the
first to the fourth layer. Its \(\EV\) learned attention is almost never uniform
(\(0.001\pm0.004\)). The prior gate remains high
(\(\beta^{\mathrm{prior}}\approx0.83\)--\(0.91\)), and the auxiliary path
accounts for roughly \(42\)--\(48\%\) of the message norm across stages and
roles. Citeseer, Cornell, and Texas have smaller auxiliary shares: across all four
\((s,R)\) combinations, \(\pi_{s,R}\in[0.21,0.25]\) on Citeseer,
\(\pi_{s,R}\in[0.11,0.22]\) on Cornell, and
\(\pi_{s,R}\in[0.12,0.27]\) on Texas. In all three cases
\(\rho_{s,R}<1\) on average, so the auxiliary messages are active but do not
dominate the main-path norm.
The prior gate also changes across stages: Citeseer and Roman-empire rely
strongly on the prior in both stages, while Cornell and Texas move more
strongly towards the GeDi prior in the \(\EV\) stage than in the \(\VE\) stage.

To test whether the auxiliary-path statistics correspond to predictive
behaviour, Table~\ref{tab:analysis-intervention} reports controlled interventions
on the same checkpoints. We close the auxiliary gate by forcing the
post-sigmoid value \(\textcolor{red}{\hat g^{(\ell)}_{e,v}}\) (in the \(\VE\) stage) or
\(\textcolor{blue}{\check g^{(\ell)}_{u,e}}\) (in the \(\EV\) stage) to zero for every
incidence and every layer, and evaluate the model again on the same test split.
We report the signed change in accuracy,
\(\Delta\)Acc., defined as the intervened accuracy minus baseline accuracy, in percentage
points; negative values therefore mean that closing the auxiliary gate hurts
performance.

Table~\ref{tab:analysis-intervention} confirms that the auxiliary path has
different predictive roles across datasets. On Telegram, closing either gate,
or both, causes large drops in test accuracy. Roman-empire is also strongly
dependent on the auxiliary path: closing both gates reduces accuracy by
\(27.7\) percentage points on average. These results show that the large
auxiliary shares in Table~\ref{tab:analysis-summary} correspond to predictive
behaviour on Telegram and Roman-empire, not merely to representational
magnitude. Citeseer, Cornell, and Texas show much smaller mean effects. In
these datasets, the auxiliary path is available and sometimes active in
magnitude, but the final predictions are less dependent on it. The intervention tells us \emph{whether} the auxiliary path matters but not
\emph{what kind of predictive effect} it produces; we probe the latter below.

\paragraph{Per-layer interventions on Roman-empire.}
Because Roman-empire is the only dataset analysed here that uses a
multi-layer PANDA configuration (four layers), the all-aux-off intervention
of Table~\ref{tab:analysis-intervention} aggregates four layerwise contributions
into a single number. The layerwise GAT-uniformity profile reported earlier
already shows that the four layers behave differently
(\(\eta_{\mathrm{unif}}^{\mathrm{gat},\VE}=0.548,\,0.088,\,0.090,\,0.010\)
from the first to the fourth layer), so we ask which layers' auxiliary paths
carry the predictive load. We run two complementary inference-time
interventions on the same ten checkpoints. \emph{Leave-one-out} (LOO)
intervention shuts off both per-incidence auxiliary gates only at one
chosen layer \(\ell\) and leaves the other three layers' auxiliary gates
active; \emph{cumulative} intervention shuts off both per-incidence auxiliary
gates at layers \(1,\ldots,k\) jointly.

\begin{table}[t!]
  \centering
  \caption{Per-layer auxiliary-gate interventions on Roman-empire (four PANDA
    layers, ten saved checkpoints). LOO shuts off both per-incidence
    auxiliary gates only at the named layer; cumulative shuts off both gates
    at layers \(1,\ldots,k\) jointly. Values are mean test accuracy \(\pm\)
    standard deviation in percent; \(\Delta\)Acc.\ is the corresponding
    accuracy change vs.\ the full model
    (\(89.2\pm0.5\)) in percentage points.}
  \label{tab:analysis-perlayer-roman}
  \small
  \setlength{\tabcolsep}{4pt}
  \begin{tabular}{llrr}
    \toprule
    Intervention & Layer(s) off    & Test acc.\ (\%) & \(\Delta\)Acc.\ (pp) \\
    \midrule
    LOO          & \(\{1\}\)       & \(81.8\pm3.0\)  & \(-7.4\pm2.8\)       \\
    LOO          & \(\{2\}\)       & \(83.3\pm1.8\)  & \(-5.8\pm1.5\)       \\
    LOO          & \(\{3\}\)       & \(85.1\pm2.8\)  & \(-4.1\pm2.4\)       \\
    LOO          & \(\{4\}\)       & \(84.9\pm3.2\)  & \(-4.3\pm3.4\)       \\
    \midrule
    Cumulative   & \(\{1\}\)       & \(81.8\pm3.0\)  & \(-7.4\pm2.8\)       \\
    Cumulative   & \(\{1,2\}\)     & \(72.5\pm4.7\)  & \(-16.7\pm4.7\)      \\
    Cumulative   & \(\{1,2,3\}\)   & \(67.3\pm4.3\)  & \(-21.9\pm4.1\)      \\
    Cumulative   & \(\{1,2,3,4\}\) & \(61.5\pm3.2\)  & \(-27.7\pm3.3\)      \\
    \bottomrule
  \end{tabular}
\end{table}

The per-layer intervention in Table~\ref{tab:analysis-perlayer-roman} reveals three things. First, the predictive role
of the auxiliary path is distributed across all four layers: every LOO drop
is non-trivial and no single layer's removal reproduces the full
\(-27.7\) point all-off drop. Second, the first layer carries the largest
LOO load (\(-7.4\) percentage points), which lines up with its having the
highest learned-attention uniformity among the four layers
(\(\eta_{\mathrm{unif}}^{\mathrm{gat},\VE}=0.548\) at layer \(1\)). This is
consistent with the auxiliary path being especially important in layers
where learned GAT attention contributes less non-uniform routing. Third, the
single-layer LOO accuracy losses sum to \(21.6\) percentage points but the
joint all-off accuracy loss is \(27.7\) points, so the intervention effects
are non-additive: the joint loss exceeds the sum of the single-layer losses.
This non-additivity is consistent with cross-layer interaction among
auxiliary paths and with partial compensation when only one layer's auxiliary
path is removed. Reading the cumulative column
makes the same point: removing layer \(2\)'s auxiliary gates \emph{after} layer \(1\)'s auxiliary
gates are already off costs an additional \(9.3\) points, well above its standalone
LOO cost of \(5.8\), so layer \(2\)'s auxiliary path becomes substantially
more critical once layer \(1\)'s is unavailable.

\subsection{Auxiliary Path Analysis}
\label{sec:analysis-neighbor-counterfactual}

Table~\ref{tab:analysis-intervention} already establishes that the auxiliary
path is predictively necessary on Telegram and Roman-empire (\(-54.0\) and
\(-27.7\) percentage-point accuracy drops when both per-incidence auxiliary
gates are shut off), and only weakly on Citeseer, Cornell, and Texas.
The analysis in this subsection therefore asks \emph{along which axis} it acts:
whether its predictive dependence is concentrated along the misleading
neighbour-majority axis, or whether it reflects a broader margin-improvement
effect spanning many class boundaries, and whether the auxiliary path mainly
rescues aux-off-wrong nodes or reinforces nodes that aux-off already predicts
correctly. All interventions in this subsection operate at inference time on a fixed trained model; they
therefore measure the trained model's predictive dependence on the auxiliary path, not the architectural necessity of having an auxiliary path at
training time. For a test node \(u\), let
\begin{equation}
  \mathcal{N}_{\mathrm{CE}}(u)
  =
  \left\{
  v\in\mathcal{V}\setminus\{u\}:
  \exists e\in\mathcal{E}
  \text{ such that }
  u,v\in H(e)\cup T(e)
  \right\}
  \label{eq:ce-neighbors}
\end{equation}
be its unique role-agnostic clique-expansion neighbours in the loaded hypergraph.
We use this only as a diagnostic neighbourhood, following the common reduction of
hyperedges to pairwise node neighbourhoods~\citep{zhou,yadati2019hypergcnnewmethodtraining,zhang2022hypergraphconvolutionalnetworksequivalency};
it does not change the directed incidence support used by PANDA. Let
\(y_u\in\{1,\ldots,C\}\) denote the ground-truth class label of node \(u\),
with \(C\) the number of classes. We define the neighbour-majority class as
\begin{equation}
  y_u^{\mathrm{nbr}}
  =
  \arg\max_c
  \sum_{v\in\mathcal{N}_{\mathrm{CE}}(u)}
  \mathbb{I}\!\left[y_v=c\right].
  \label{eq:neighbor-majority-class}
\end{equation}
Ties in the \(\arg\max\) are broken by the fixed integer ordering of the
class labels. The neighbour-majority class \(y_u^{\mathrm{nbr}}\) is defined
only on nodes with \(\mathcal{N}_{\mathrm{CE}}(u)\neq\emptyset\); on the
analysed datasets, every test node has at least one CE neighbour except for
Citeseer, where roughly \(1.4\%\) of test nodes per run are isolated and are
excluded from the \(y_u^{\mathrm{nbr}}\)-referencing rates below. 

Let \(o_{u,c}^{\mathrm{full}}\) be the output logit assigned by the original model
to class \(c\) for node \(u\), and let \(o_{u,c}^{\mathrm{aux\text{-}off}}\) be the
corresponding output logit obtained by forcing both per-incidence auxiliary
gates to zero, i.e.,
\(\textcolor{red}{\hat g^{(\ell)}_{e,v}}=0\) and \(\textcolor{blue}{\check g^{(\ell)}_{u,e}}=0\) at every layer
and every incidence (the same intervention used in
Table~\ref{tab:analysis-intervention}). The corresponding per-node predictions
are \(\hat y_u^{\mathrm{full}}=\arg\max_c o_{u,c}^{\mathrm{full}}\) and
\(\hat y_u^{\mathrm{aux\text{-}off}}=\arg\max_c o_{u,c}^{\mathrm{aux\text{-}off}}\) (ties broken
by the same fixed class ordering used for \(y_u^{\mathrm{nbr}}\)).
Let \(\mathcal{T}\) denote the set of test nodes for the run, and let
\(\mathcal{T}_{\text{CE}}=\{u\in\mathcal{T}:\mathcal{N}_{\mathrm{CE}}(u)\neq\emptyset\}\)
be the subset on which \(y_u^{\mathrm{nbr}}\) is defined. All quantities
below are sample statistics computed over \(\mathcal{T}\) or
\(\mathcal{T}_{\text{CE}}\) depending on whether they reference
\(y_u^{\mathrm{nbr}}\). For any node-level event \(E(u)\) and any subset
\(\mathcal{S}\subseteq\mathcal{T}\) we write
\[
  \widehat{\Pr}_{u\in\mathcal{S}}\!\left[E(u)\right]
  =
  \frac{1}{|\mathcal{S}|}\sum_{u\in\mathcal{S}}\mathbb{I}\!\left[E(u)\right],
\]
and define the test accuracies
\[
  \mathrm{Acc}^{\mathrm{full}}
  =
  \widehat{\Pr}_{u\in\mathcal{T}}\!\left[\hat y_u^{\mathrm{full}}=y_u\right],
  \qquad
  \mathrm{Acc}^{\mathrm{aux\text{-}off}}
  =
  \widehat{\Pr}_{u\in\mathcal{T}}\!\left[\hat y_u^{\mathrm{aux\text{-}off}}=y_u\right].
\]

Table~\ref{tab:neighbor-counterfactual} reports the following quantities. The accuracy-gap column is
\[
  \Delta\mathrm{Acc.}
  =
  \mathrm{Acc}^{\mathrm{aux\text{-}off}}-\mathrm{Acc}^{\mathrm{full}},
\]
expressed in percentage points, so that negative values mean that shutting
off the auxiliary path hurts accuracy, matching the sign convention of
Table~\ref{tab:analysis-intervention}. The full-model accuracy is reported
in Table~\ref{tab:analysis-intervention}, and
\(\mathrm{Acc}^{\mathrm{aux\text{-}off}}\) is obtained by applying the
``both gates \(=0\)'' accuracy change to the corresponding full-model
accuracy. The column
\(q_{\mathrm{mis}}\) is the fraction of test nodes for which the
neighbour-majority class is misleading:
\[
  q_{\mathrm{mis}}
  =
  \widehat{\Pr}_{u\in\mathcal{T}_{\text{CE}}}\!\left[y_u^{\mathrm{nbr}}\neq y_u\right].
\]
The class-agnostic flip rates are
\[
  q_{\mathrm{help}}
  =
  \widehat{\Pr}_{u\in\mathcal{T}}\!\left[
  \hat y_u^{\mathrm{aux\text{-}off}}\neq y_u
  \ \wedge\
  \hat y_u^{\mathrm{full}}=y_u
  \right]
  \quad
  q_{\mathrm{hurt}}
  =
  \widehat{\Pr}_{u\in\mathcal{T}}\!\left[
  \hat y_u^{\mathrm{aux\text{-}off}}=y_u
  \ \wedge\
  \hat y_u^{\mathrm{full}}\neq y_u
  \right],
\]
which measure the fractions of test nodes where the auxiliary path flips
the prediction onto or off of the true class, regardless of the
neighbour-majority class. By construction,
\(\Delta\mathrm{Acc.}=q_{\mathrm{hurt}}-q_{\mathrm{help}}\),
so these two rates decompose the accuracy gap into a gross-help and
gross-hurt component. The column \(q_{\mathrm{nbr\text{-}rescue}}\) counts the
test nodes for which the aux-off model follows the misleading
neighbour-majority class while the full model predicts the true class,
normalised over the full test set. Writing
\(\mathcal{T}_{\neq}=\{u\in\mathcal{T}_{\text{CE}}:y_u^{\mathrm{nbr}}\neq y_u\}\),
\[
q_{\mathrm{nbr\text{-}rescue}}
  =\widehat{\Pr}_{u\in\mathcal{T}}\!\left[
  u\in\mathcal{T}_{\neq}
  \ \wedge\
  \hat y_u^{\mathrm{aux\text{-}off}}=y_u^{\mathrm{nbr}}
  \ \wedge\
  \hat y_u^{\mathrm{full}}=y_u
  \right].
\]
The ratio between \(q_{\mathrm{nbr\text{-}rescue}}\) and \(q_{\mathrm{help}}\) measures the share of helped
flips attributable to strict neighbour-majority rescue.

Finally, we measure the logit-margin shift from the neighbour-majority class to the true one induced by the auxiliary path:
\begin{equation}
  \Delta_u^{\mathrm{nbr}}
  =
  \left(
  o_{u,y_u}^{\mathrm{full}}
  -
  o_{u,y_u^{\mathrm{nbr}}}^{\mathrm{full}}
  \right)
  -
  \left(
  o_{u,y_u}^{\mathrm{aux\text{-}off}}
  -
  o_{u,y_u^{\mathrm{nbr}}}^{\mathrm{aux\text{-}off}}
  \right).
  \label{eq:neighbor-margin-shift}
\end{equation}
\(\Delta_u^{\mathrm{nbr}}\) is identically zero when \(y_u=y_u^{\mathrm{nbr}}\)
(both terms compare a class logit with itself), so we report only its mean
on \(\mathcal{T}_{\neq}\):
\[
  \overline{\Delta}^{\mathrm{nbr}}_{\neq}
  =
  \frac{1}{|\mathcal{T}_{\neq}|}
  \sum_{u\in\mathcal{T}_{\neq}}
  \Delta_u^{\mathrm{nbr}}.
\]
A positive \(\overline{\Delta}^{\mathrm{nbr}}_{\neq}\) means that, on average,
the auxiliary path shifts the true-vs-neighbour-majority logit difference in
favour of the true class. This is a directional statement, not an absolute
one: the difference may still be negative under both models (neighbour-majority
still wins, by a smaller gap), flip sign, or grow from positive to more
positive --- the shift is detectable even when the final argmax does not
change.

\begin{table*}[t]
  \centering
  \caption{Neighbour-majority analysis of the auxiliary path.
  \(\mathcal H_{\mathrm{CE}}\) is the dataset-level
  clique-expansion homophily reported by \citet{wang2023hypergraph}.
  \(q_{\mathrm{nbr\text{-}rescue}}\) is the unconditional fraction of test
  nodes (over \(|\mathcal T|\)) for which aux-off follows the misleading
  neighbour-majority class and the full model predicts the true class.
  \(\overline{\Delta}^{\mathrm{nbr}}_{\neq}\) is the true-vs-neighbour-majority
  logit-margin shift on nodes with \(y_u^{\mathrm{nbr}}\neq y_u\).}
  \label{tab:neighbor-counterfactual}
  \small
  \setlength{\tabcolsep}{3pt}
  \begin{tabular*}{\textwidth}{@{\extracolsep{\fill}}rrrrrrrr@{}}
    \toprule
    Dataset
    & \(\mathcal H_{\mathrm{CE}}\)
    & \(\Delta\)Acc.
    & \(q_{\mathrm{mis}}\)
    & \(q_{\mathrm{help}}\)
    & \(q_{\mathrm{hurt}}\)
    & \(q_{\mathrm{nbr\text{-}rescue}}\)
    & \(\overline{\Delta}^{\mathrm{nbr}}_{\neq}\) \\
    \midrule
    Citeseer      & 0.795 & \(-1.0\)            & 23.97 & 1.84  & 0.82 & 0.22  & \(-0.28\) \\
    Cornell       & 0.420 & \(-4.6\)            & 57.84 & 5.95  & 1.35 & 1.08  & \(+0.09\) \\
    Roman-empire  & 0.236 & \(-27.7\)           & 95.68 & 29.44 & 1.77 & 2.20  & \(+5.29\) \\
    Telegram      & 0.285 & \(-54.0\)           & 81.73 & 56.54 & 2.50 & 16.35 & \(+3.58\) \\
    Texas         & 0.505 & \(\phantom{-}0.0\)  & 47.03 & 3.24  & 3.24 & 1.35  & \(+0.09\) \\
    \bottomrule
  \end{tabular*}
\end{table*}

Table~\ref{tab:neighbor-counterfactual} separates a margin-level
neighbour-majority effect (\(\overline{\Delta}^{\mathrm{nbr}}_{\neq}\)) from a
strict argmax-level rescue effect (\(q_{\mathrm{nbr\text{-}rescue}}\)), and
shows that the auxiliary path's behaviour on the neighbour axis is
dataset-dependent and qualitatively aligned with clique-expansion homophily. On
Roman-empire and Telegram, where \(\mathcal H_{\mathrm{CE}}\) is low and
\(q_{\mathrm{mis}}\) is high (\(95.7\%\) and \(81.7\%\)), the auxiliary path
strongly increases the true-vs-neighbour-majority margin
(\(\overline{\Delta}^{\mathrm{nbr}}_{\neq}=+5.29\) and \(+3.58\)). On these
low-homophily datasets, where the clique-expansion neighbourhood is often
unreliable, the auxiliary path therefore acts as a counter-neighbourhood
signal.

Strict neighbour-majority rescue, however, explains only part of that gain.
On Roman-empire \(q_{\mathrm{nbr\text{-}rescue}}=2.20\%\) is far below
\(q_{\mathrm{help}}=29.44\%\) (\(2.20/29.44\approx 7.5\%\); about \(125\) of
\(5666\) test nodes per run), and on Telegram the corresponding values are
\(16.35\%\) and \(56.54\%\) (\(\approx 28.9\%\)). The auxiliary path therefore
does not merely correct predictions that follow the misleading
neighbour-majority class; its main predictive role, especially on
Roman-empire, is broader margin improvement against many competing classes.

Citeseer shows the opposite pattern. It is the most homophilic dataset
(\(\mathcal H_{\mathrm{CE}}=0.795\)), the auxiliary intervention has only a
small effect (\(\Delta\mathrm{Acc.}=-1.0\)), and
\(\overline{\Delta}^{\mathrm{nbr}}_{\neq}=-0.28\): on the rare misleading-neighbour
nodes, the auxiliary path mildly leans \emph{with} the neighbour-majority,
consistent with the fact that the neighbour-majority is usually informative on
homophilic graphs. The auxiliary path is therefore not universally
anti-neighbour-majority; in these experiments it pushes against the
neighbour-majority mainly on the datasets where the neighbourhood is unreliable. Cornell and Texas sit between these regimes
with mid-range \(\mathcal H_{\mathrm{CE}}\) and small intervention effects, so
their \(\overline{\Delta}^{\mathrm{nbr}}_{\neq}\) values are close to zero.

\subsubsection*{Generalised margin shift \(\Delta_u^{\mathrm{max}}\) and conditional split by aux-off correctness}

The fact that \(q_{\mathrm{nbr\text{-}rescue}}\) is much smaller than
\(q_{\mathrm{help}}\), especially on Roman-empire, motivates a generalised
margin analysis: if the auxiliary path is not only correcting
neighbour-majority mistakes, then it should also improve the true-class margin
against non-neighbour-majority competitors. The neighbour-majority margin
\(\overline{\Delta}^{\mathrm{nbr}}_{\neq}\) focuses on one specific axis,
the true class versus the neighbour-majority class. To also cover corrections
whose strongest competitor is not the neighbour-majority class, we report a
second margin shift defined against the best non-true class at each node:
\begin{equation}
  \Delta_u^{\mathrm{max}}
  =
  \left(
  o_{u,y_u}^{\mathrm{full}}
  -
  \max_{c\neq y_u} o_{u,c}^{\mathrm{full}}
  \right)
  -
  \left(
  o_{u,y_u}^{\mathrm{aux\text{-}off}}
  -
  \max_{c\neq y_u} o_{u,c}^{\mathrm{aux\text{-}off}}
  \right).
  \label{eq:max-other-margin-shift}
\end{equation}
The competitor class is chosen independently of the neighbour-majority and
independently in the full and aux-off logits, so
\(\Delta_u^{\mathrm{max}}\) is defined for every test node and is positive
whenever the auxiliary path shifts the logit difference between the true
class and its strongest competitor in favour of the true class. We report its sample mean over all test nodes:
\[
  \overline{\Delta}^{\mathrm{max}}
  =
  \frac{1}{|\mathcal{T}|}
  \sum_{u\in\mathcal{T}}
  \Delta_u^{\mathrm{max}}.
\]
We focus on two
distinct services the auxiliary path can render: \emph{reinforcing} the
true-class margin on nodes the aux-off model already predicts correctly,
and \emph{rescuing} nodes the aux-off model gets wrong by closing (or
flipping) their true-vs-best-other gap. Splitting the test set by the aux-off
prediction's correctness separates the two. Let
\(\mathcal{T}^{\mathrm{aux\text{-}off}}_{\checkmark}=\{u\in\mathcal{T}:\hat
y_u^{\mathrm{aux\text{-}off}}=y_u\}\) and
\(\mathcal{T}^{\mathrm{aux\text{-}off}}_{\times}=\mathcal{T}\setminus
\mathcal{T}^{\mathrm{aux\text{-}off}}_{\checkmark}\), and write
\(n_{\checkmark}=|\mathcal{T}^{\mathrm{aux\text{-}off}}_{\checkmark}|\)
and
\(n_{\times}=|\mathcal{T}^{\mathrm{aux\text{-}off}}_{\times}|\) for the
corresponding per-checkpoint subset sizes (so
\(n_{\checkmark}+n_{\times}=|\mathcal{T}|\) by construction).
Table~\ref{tab:delta-max-conditional} reports the mean of
\(\Delta_u^{\mathrm{max}}\) and the fraction of positive
\(\Delta_u^{\mathrm{max}}\) on each subset, separately for each dataset;
the \(n\) columns of that table report the checkpoint-mean of
\(n_{\checkmark}\) and \(n_{\times}\).

\begin{table}[t]
  \centering
  \caption{Conditional split of \(\Delta_u^{\mathrm{max}}\) by aux-off
    correctness. \(n\) is the mean subset size per checkpoint; \(\%>0\) is the
    fraction of nodes in the subset with positive \(\Delta_u^{\mathrm{max}}\).
    \(n_{\mathrm{hurt}}\) counts aux-off-correct nodes flipped to wrong by the
    full model, while \(n_{\mathrm{help}}\) counts aux-off-wrong nodes flipped to
    correct. All columns are means over ten saved checkpoints.}
  \label{tab:delta-max-conditional}
  \small
  \setlength{\tabcolsep}{4pt}
  \begin{tabular*}{\textwidth}{@{\extracolsep{\fill}}lrrrrrrrr@{}}
    \toprule
    & \multicolumn{4}{c}{\(u\in\mathcal{T}^{\mathrm{aux\text{-}off}}_{\checkmark}\)
    (aux-off correct)}
    & \multicolumn{4}{c}{\(u\in\mathcal{T}^{\mathrm{aux\text{-}off}}_{\times}\)
    (aux-off wrong)} \\
    \cmidrule(lr){2-5}\cmidrule(lr){6-9}
    Dataset
    & \(n\) & \(\overline{\Delta}^{\mathrm{max}}\) & \(\%>0\)
    & \(n_{\mathrm{hurt}}\)
    & \(n\) & \(\overline{\Delta}^{\mathrm{max}}\) & \(\%>0\)
    & \(n_{\mathrm{help}}\) \\
    \midrule
    Citeseer      & 618  & \(+0.19\)  & \(69.6\) & \(6.8\)   & 210  & \(+0.09\)  & \(41.9\) & \(15.2\)   \\
    Cornell       & 26   & \(+0.75\)  & \(71.7\) & \(0.5\)   & 11   & \(+0.58\)  & \(55.3\) & \(2.2\)    \\
    Roman-empire  & 3484 & \(+2.99\)  & \(83.3\) & \(100.5\) & 2182 & \(+13.97\) & \(91.1\) & \(1668.0\) \\
    Telegram      & 18   & \(+0.14\)  & \(61.5\) & \(1.3\)   & 34   & \(+9.39\)  & \(98.5\) & \(29.4\)   \\
    Texas         & 31   & \(+0.03\)  & \(50.8\) & \(1.2\)   & 6    & \(+0.42\)  & \(50.4\) & \(1.2\)    \\
    \bottomrule
  \end{tabular*}
\end{table}

Two interpretive notes apply throughout. \(\Delta_u^{\mathrm{max}}\) is a
continuous logit-margin shift, not an argmax-flip indicator: on the
aux-off-wrong subset a positive \(\Delta_u^{\mathrm{max}}\) is necessary
for a wrong-to-right flip (the true-vs-best-other gap must cross zero) but
not sufficient -- the gap can shrink without crossing it, leaving the
prediction wrong with a smaller margin. The \(\%\!>\!0\) column on the
aux-off-wrong subset is therefore an upper bound on the actual
flip-to-correct rate, which is given by \(n_{\mathrm{help}}/n_{\times}\).
On the aux-off-correct subset there is no such ambiguity for the
reinforcement reading: aux-off-correct implies that the true-class logit
is maximal under the tie-breaking rule, which (ignoring measure-zero ties)
corresponds to
\((o_{u,y_u}-\max_{c\neq y_u}o_{u,c})^{\mathrm{aux\text{-}off}}>0\), so a
positive \(\Delta_u^{\mathrm{max}}\) means the gap is even larger and the
prediction stays correct; \(n_{\mathrm{hurt}}/n_{\checkmark}\) gives the
opposite-direction flip rate (correct to wrong).

Table~\ref{tab:delta-max-conditional} is intended as a rescue-versus-reinforcement diagnostic rather than additional evidence that the auxiliary path is useful, since that case is already made by Table~\ref{tab:analysis-intervention}. The two subsets answer different questions: the aux-off-correct subset shows whether the auxiliary path widens already-positive
gaps (reinforcement), while the aux-off-wrong subset shows whether it can flip negative gaps to positive (rescue). Both effects can coexist, and their relative size is what we read   
off the table.
  
Roman-empire is the clearest case where both mechanisms are active. On the aux-off-correct subset (\(n_{\checkmark}=3484\)), the per-node mean shift is \(\overline{\Delta}^{\mathrm{max}}=+2.99\) with \(83.3\%\) positive and only \(\sim\!100\) hurt nodes per checkpoint, so reinforcement is real but moderate in magnitude. On the
aux-off-wrong subset (\(n_{\times}=2182\)), the per-node mean shift is much larger (\(+13.97\)), \(91.1\%\) positive, and \(1668\) of those \(2182\) nodes flip to correct (\(n_{\mathrm{help}}/n_{\times}\approx 76\%\)). The total accuracy contribution is therefore dominated by rescue: the per-node shift is roughly \(5\times\) larger on aux-off-wrong
than on aux-off-correct, and \(1668\) rescued nodes vastly outnumber the \({\sim}100\) hurt by reinforcement.

Telegram is sharper: rescue is essentially the only mechanism. The aux-off-correct subset has only \(n_{\checkmark}=18\) nodes per checkpoint with a small mean shift (\(+0.14\)) and a
near-coin-flip \(61.5\%\) positive -- there is little reinforcement to do because there are few aux-off-correct nodes to begin with. The aux-off-wrong subset has \(n_{\times}=34\)
nodes per checkpoint with mean shift \(+9.39\), \(98.5\%\) positive, and \(29.4\) of those \(34\) nodes flip to correct (\(\approx 86\%\)). The auxiliary path is acting almost        
exclusively as a rescue mechanism on Telegram.

Citeseer does not show a consistent rescue signal. On its aux-off-wrong subset (\(n_{\times}=210\)), only \(41.9\%\) of nodes have a positive shift -- below \(50\%\), so on average   
the auxiliary path shrinks the gap on as many wrong nodes as it widens it. The mean shift (\(+0.09\)) is small on the logit scale, and only \(15.2\) of \(210\) wrong nodes flip to
correct (\({\approx}7\%\)). This is consistent with Citeseer's small intervention effect in Table~\ref{tab:analysis-intervention}: the auxiliary path is active in magnitude but does  
not translate into coherent rescue behaviour.

Cornell and Texas have too few aux-off-wrong nodes per checkpoint (\(n_{\times}\approx 11\) and \(n_{\times}\approx 6\) respectively) for the rescue percentages to be statistically   
meaningful -- a single misclassified node moves the rate by \({\sim}10\text{--}17\) percentage points. The Cornell numbers are qualitatively in line with Roman-empire (positive shift
in both subsets) but we treat them as suggestive at best. Texas is closer to noise: \(\%\!>\!0\) hovers around \(50\%\) on both subsets and per-node shifts are an order of magnitude  
smaller than on the larger datasets.

\subsection{Pass-wise geometry of the auxiliary path}
\label{sec:analysis-cos-pass}

The previous analyses characterise \emph{whether} the auxiliary path matters
and \emph{what kind} of predictive effect it produces. We next ask a simpler
geometric question: within each message-passing pass, is the auxiliary
representation aligned with the corresponding main-path representation, or does
it provide a distinct direction? We measure this from two complementary angles.
The \emph{source-side} cosine compares the main and auxiliary representations
of each source before message passing. The \emph{receiver-side} cosine compares
the aggregated main and auxiliary contributions arriving at each destination
after attention-weighted summation. The two answer different questions: the source-side characterises the
intrinsic geometry of the auxiliary representation relative to the main one at
each source, while the receiver-side characterises the geometry of the
contributions actually mixed at the destination, where per-incidence weights
and cross-source interference can either reinforce or wash out the per-source
angle.

\paragraph{Source-side cosine.}
For the \(\VE\) stage at layer \(\ell\), the source is a node \(v\), with main
and auxiliary representations \(\mathbf{x}_v^{\mathrm{main},(\ell)}\) and
\(\mathbf{x}_v^{\mathrm{aux},(\ell)}\). For the \(\EV\) stage, the source is a
hyperedge \(e\), with \(\mathbf{h}_e^{\mathrm{main},(\ell)}\) and
\(\mathbf{h}_e^{\mathrm{aux},(\ell)}\). The source-side cosine is
\begin{equation}
  c_s^{\mathrm{src},(\ell,p)}
  =
  \begin{cases}
    \cos\!\left(
    \mathbf{x}_v^{\mathrm{main},(\ell)},
    \mathbf{x}_v^{\mathrm{aux},(\ell)}
    \right)
     &
    \text{if } p=\VE,\; s=v\in\mathcal V,
    \\[2pt]
    \cos\!\left(
    \mathbf{h}_e^{\mathrm{main},(\ell)},
    \mathbf{h}_e^{\mathrm{aux},(\ell)}
    \right)
     &
    \text{if } p=\EV,\; s=e\in\mathcal E.
  \end{cases}
  \label{eq:cos-source-aux}
\end{equation}

\paragraph{Receiver-side cosine.}
Using the per-incidence main and auxiliary contributions of
Eq.~\eqref{eq:analysis-messages}, define the per-receiver aggregates
\begin{equation}
  \begin{aligned}
    \textcolor{cyan}{\hat m_e^{\mathrm{main},(\ell)}}
     & =
    \sum_{v\in H(e)\cup T(e)}
    \textcolor{cyan}{\hat\alpha_{e,v}^{\mathrm{main}}}\,\mathbf{x}_v^{\mathrm{main},(\ell)},
     &
    \textcolor{red}{\hat m_e^{\mathrm{aux},(\ell)}}
     & =
    \sum_{v\in H(e)\cup T(e)}
    \textcolor{red}{\hat\alpha_{e,v}^{\mathrm{aux}}}\,\mathbf{x}_v^{\mathrm{aux},(\ell)},
    \\
    \textcolor{orange}{\check m_u^{\mathrm{main},(\ell)}}
     & =
    \sum_{e\in\Gamma(u)}
    \textcolor{orange}{\check\alpha_{u,e}^{\mathrm{main}}}\,\mathbf{h}_e^{\mathrm{main},(\ell)},
     &
    \textcolor{blue}{\check m_u^{\mathrm{aux},(\ell)}}
     & =
    \sum_{e\in\Gamma(u)}
    \textcolor{blue}{\check\alpha_{u,e}^{\mathrm{aux}}}\,\mathbf{h}_e^{\mathrm{aux},(\ell)} ,
  \end{aligned}
  \label{eq:receiver-aggregates}
\end{equation}
The receiver-side cosine is then
\begin{equation}
  c_r^{\mathrm{recv},(\ell,p)}
  =
  \begin{cases}
    \cos\!\left(
    \textcolor{cyan}{\hat m_e^{\mathrm{main},(\ell)}},\,
    \textcolor{red}{\hat m_e^{\mathrm{aux},(\ell)}}
    \right)
     &
    \text{if } p=\VE,\; r=e\in\mathcal E,
    \\[2pt]
    \cos\!\left(
    \textcolor{orange}{\check m_u^{\mathrm{main},(\ell)}},\,
    \textcolor{blue}{\check m_u^{\mathrm{aux},(\ell)}}
    \right)
     &
    \text{if } p=\EV,\; r=u\in\mathcal V.
  \end{cases}
  \label{eq:cos-receiver-aux}
\end{equation}

Table~\ref{tab:analysis-cos-pass} reports the mean of both cosines across all
sources or receivers and ten saved checkpoints, grouped by dataset, layer, and
pass; Figures~\ref{fig:analysis-cos-hist} and~\ref{fig:analysis-cos-recv-hist} reports the distribution of the
source- and receiver-side cosine, respectively.

\begin{table}[t]
  \centering
  \caption{Source-side and receiver-side cosines between the main-path and
    auxiliary representations,
    Eqs.~\eqref{eq:cos-source-aux},~\eqref{eq:cos-receiver-aux}. Source-side
    means are computed across all sources of the pass (nodes for \(\VE\),
    hyperedges for \(\EV\)) and ten saved checkpoints; receiver-side means
    are computed across all receivers of the pass (hyperedges for \(\VE\),
    nodes for \(\EV\)) and ten checkpoints. Standard deviations across pooled
    source--checkpoint or receiver--checkpoint instances are shown in
    parentheses. Roman-empire is the only multi-layer configuration in this
    analysis.}
  \label{tab:analysis-cos-pass}
  \small
  \setlength{\tabcolsep}{4pt}
  \begin{tabular*}{\textwidth}{@{\extracolsep{\fill}}llcccc@{}}
    \toprule
    &
    & \multicolumn{2}{c}{\(\VE\)}
    & \multicolumn{2}{c}{\(\EV\)} \\
    \cmidrule(lr){3-4} \cmidrule(lr){5-6}
    Dataset & Layer
    & \(\mathrm{Avg.}\; c_s^{\mathrm{src}}\) (std)
    & \(\mathrm{Avg.}\; c_r^{\mathrm{recv}}\) (std)
    & \(\mathrm{Avg.}\; c_s^{\mathrm{src}}\) (std)
    & \(\mathrm{Avg.}\; c_r^{\mathrm{recv}}\) (std) \\
    \midrule
    Citeseer & 0 & \(+0.375\,(0.099)\) & \(+0.300\,(0.094)\) & \(+0.006\,(0.124)\) & \(+0.025\,(0.114)\) \\
    Cornell  & 0 & \(+0.546\,(0.130)\) & \(+0.190\,(0.091)\) & \(+0.027\,(0.131)\) & \(+0.090\,(0.119)\) \\
    Texas    & 0 & \(+0.490\,(0.145)\) & \(+0.211\,(0.114)\) & \(+0.035\,(0.129)\) & \(+0.086\,(0.094)\) \\
    Telegram & 0 & \(+0.138\,(0.113)\) & \(+0.036\,(0.101)\) & \(-0.023\,(0.171)\) & \(-0.044\,(0.166)\) \\
    \midrule
    \multirow{4}{*}{Roman-empire}
    & 0 & \(+0.529\,(0.106)\) & \(+0.389\,(0.077)\) & \(+0.125\,(0.114)\) & \(+0.006\,(0.103)\) \\
    & 1 & \(+0.324\,(0.140)\) & \(+0.217\,(0.108)\) & \(-0.017\,(0.137)\) & \(-0.035\,(0.092)\) \\
    & 2 & \(+0.213\,(0.149)\) & \(+0.172\,(0.113)\) & \(-0.160\,(0.174)\) & \(-0.087\,(0.133)\) \\
    & 3 & \(+0.146\,(0.188)\) & \(+0.130\,(0.165)\) & \(-0.082\,(0.178)\) & \(-0.057\,(0.128)\) \\
    \bottomrule
  \end{tabular*}
\end{table}

\paragraph{\(\VE\) keeps a positive component of the main direction.}
The node-to-hyperedge pass has positive mean source-side cosine in every
dataset and layer. On Citeseer, Cornell, Texas, and the first Roman-empire
layer, the source-side cosine is moderately positive (\(0.38\)--\(0.55\)),
suggesting that the auxiliary node representation preserves a sizeable
component of the corresponding main-path direction. The receiver-side cosine
is also positive throughout but smaller in magnitude (e.g.\ Cornell
\(0.55\to0.19\), Texas \(0.49\to0.21\), Roman-empire layer~0
\(0.53\to0.39\)). Aggregation thus dilutes but does not invert the alignment:
the auxiliary contributions still arrive at hyperedges with a positive
component along the main aggregate, just less strongly than the per-source
geometry would suggest.

Telegram and the deeper Roman-empire layers are different. Telegram has a much
smaller source-side \(\VE\) cosine (\(+0.138\)), and Roman-empire decreases
from \(+0.529\) at layer \(0\) to \(+0.146\) at layer \(3\). These are also
the settings where the auxiliary path is most predictively important. Thus,
when the auxiliary path carries substantial predictive load, its
node-to-hyperedge representation becomes less aligned with the main path and
more geometrically distinct.

\paragraph{\(\EV\) is nearly orthogonal at both ends.}
The hyperedge-to-node pass is centred near zero on the single-layer datasets
at the source side and stays close to zero at the receiver side: Citeseer,
Cornell, Texas, and Telegram all have receiver-side cosines in
\([-0.05,+0.10]\). The auxiliary hyperedge representation is therefore neither
a rescaled copy of the main one at the source, nor does aggregation
reintroduce alignment at the receiver: at both ends of the pass the auxiliary
contribution provides a distinct feature direction.

Roman-empire makes the picture sharper across layers. The source-side \(\EV\)
cosine moves from mildly positive at layer \(0\) (\(+0.125\)) to near-zero at
layer \(1\), and becomes negative in layers \(2\) and \(3\) (\(-0.160\) and
\(-0.082\)). The receiver-side cosine consistently sits closer to zero than
the source-side: layer~0 collapses from \(+0.125\) to \(+0.006\), and layer~2
moves from \(-0.160\) to \(-0.087\). This matches the intervention results:
Roman-empire relies on the auxiliary path across multiple layers, and the
auxiliary contribution arrives at each node along a direction that is
essentially orthogonal to the main aggregate, mildly tilted away from it in
the deeper layers.

\paragraph{Interpretation.}
These cosines support the view that the auxiliary path is not merely a
duplicate message stream. In the \(\VE\) pass, the auxiliary representation
preserves part of the main node direction at the source and at the receiver,
but this alignment weakens in Telegram and in deeper Roman-empire layers,
where the auxiliary path is most important. In the \(\EV\) pass, the
auxiliary representation is nearly orthogonal at the source, and aggregation
generally tightens rather than loosens this orthogonality, so the auxiliary
contribution at each node provides a geometrically distinct channel that
cannot be obtained by rescaling the main aggregate.

\begin{figure}[htb!]
  \centering
  \includegraphics[width=\textwidth]{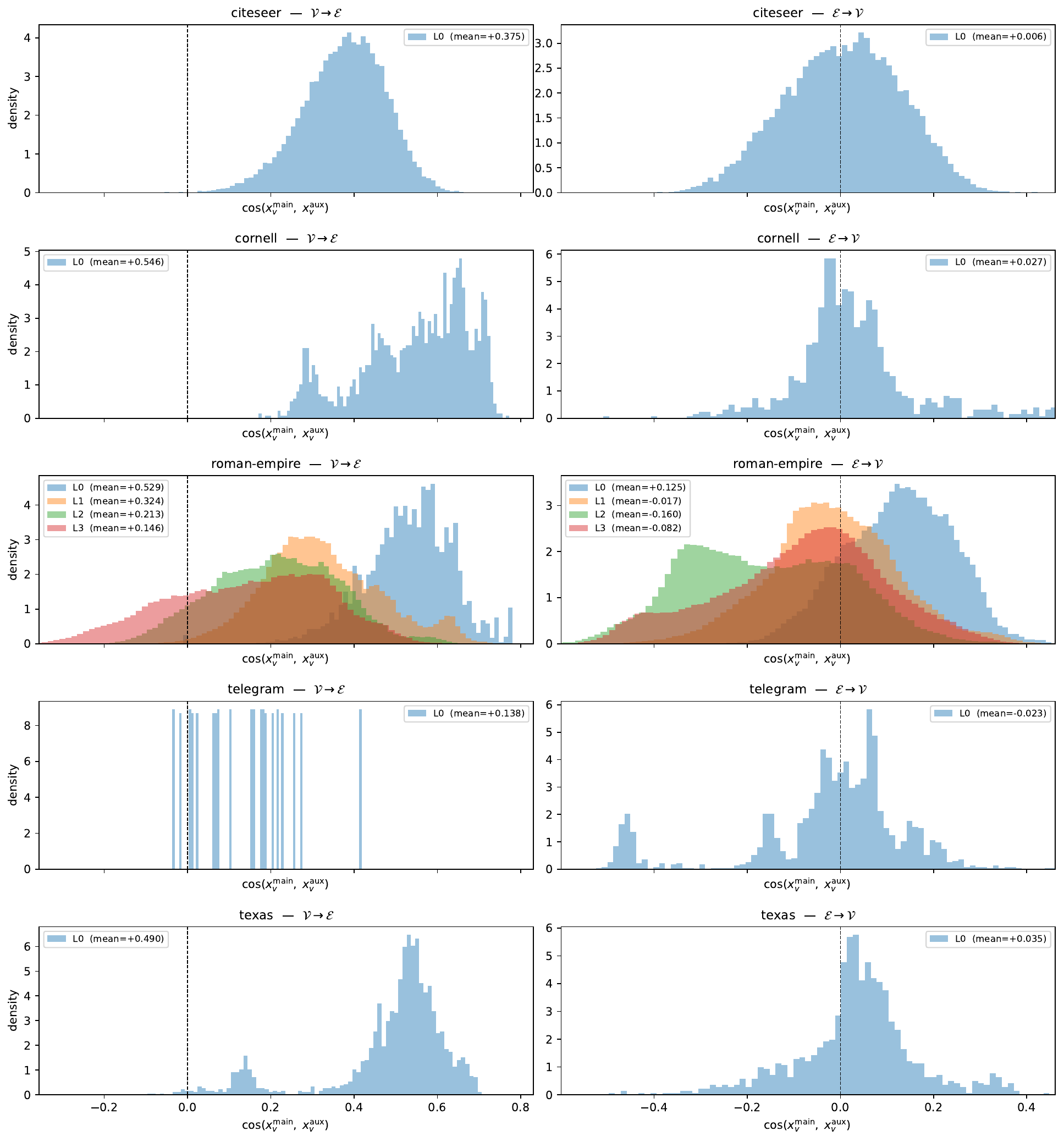}
  \caption{Source-wise distribution of \(c_s^{(\ell,p)}\)
  (Eq.~\eqref{eq:cos-source-aux}) across all sources and ten saved
  checkpoints, overlaid by layer for the multi-layer Roman-empire family.
  The vertical dashed line marks \(c=0\). The \(\VE\) cosine is uniformly
  positive across families and layers; the \(\EV\) cosine is concentrated
  near zero on Citeseer, Cornell, Texas, and Telegram, and spreads more
  widely on Roman-empire (with layer~2 mass clearly to the left of zero).}
  \label{fig:analysis-cos-hist}
\end{figure}

\begin{figure*}[t]
  \centering
  \includegraphics[width=\textwidth]{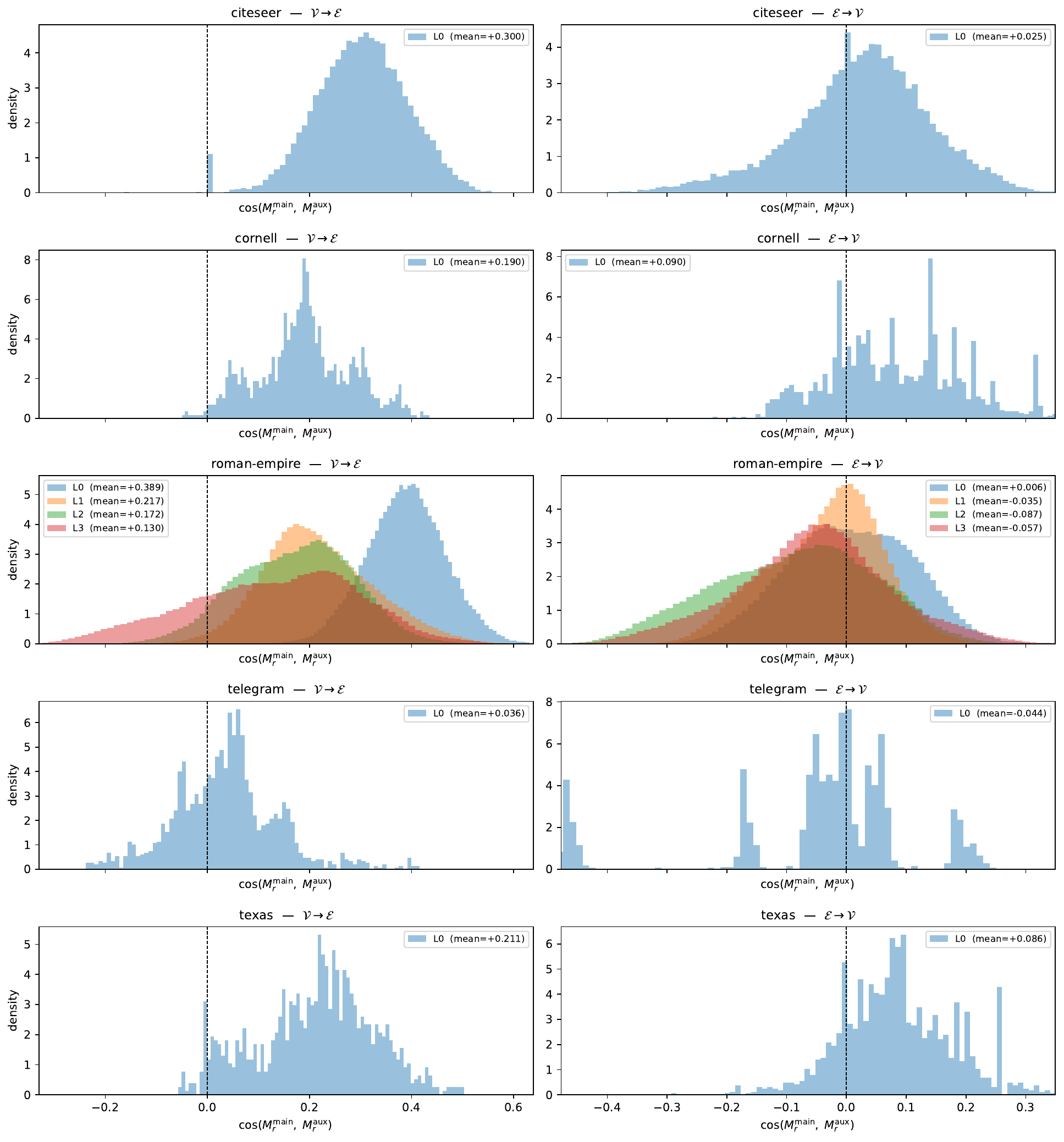}
  \caption{Receiver-side distribution of \(c_r^{\mathrm{recv},(\ell,p)}\)
  (Eq.~\eqref{eq:cos-receiver-aux}) across all receivers and ten saved
  checkpoints, overlaid by layer for the multi-layer Roman-empire family.
  The vertical dashed line marks \(c=0\). The \(\VE\) receiver-side cosine
  is positive throughout but smaller in magnitude than its source-side
  counterpart; the \(\EV\) receiver-side cosine is concentrated near zero
  on every dataset and layer, including Roman-empire's deeper layers, where
  aggregation tightens the source-side near-orthogonality further.}
  \label{fig:analysis-cos-recv-hist}
\end{figure*}

\section*{NeurIPS Paper Checklist}

\begin{enumerate}

\item {\bf Claims}
    \item[] Question: Do the main claims made in the abstract and introduction accurately reflect the paper's contributions and scope?
    \item[] Answer: \answerYes{} % Replace by \answerYes{}, \answerNo{}, or \answerNA{}.
    \item[] Justification: Section \ref{sec:results} and Table \ref{tab:eval-merged} show our quantitative evaluation results, supporting our performance claims. We also provide Theorems \ref{thm:gedi-recovery} and \ref{thm:gedi-star-expansion}, with proofs in Appendix \ref{sec:appendix-proofs}.

\item {\bf Limitations}
    \item[] Question: Does the paper discuss the limitations of the work performed by the authors?
    \item[] Answer: \answerYes{} % Replace by \answerYes{}, \answerNo{}, or \answerNA{}.
    \item[] Justification: We discuss the limitations within our result analysis in Section \ref{sec:results} and the conclusion in Section \ref{sec:conclusions}.

\item {\bf Theory assumptions and proofs}
    \item[] Question: For each theoretical result, does the paper provide the full set of assumptions and a complete (and correct) proof?
    \item[] Answer: \answerYes{} % Replace by \answerYes{}, \answerNo{}, or \answerNA{}.
    \item[] Justification: We provide proofs for our theorems in Appendix \ref{sec:appendix-proofs}.

    \item {\bf Experimental result reproducibility}
    \item[] Question: Does the paper fully disclose all the information needed to reproduce the main experimental results of the paper to the extent that it affects the main claims and/or conclusions of the paper (regardless of whether the code and data are provided or not)?
    \item[] Answer: \answerYes{} % Replace by \answerYes{}, \answerNo{}, or \answerNA{}.
    \item[] Justification: We provide complete details about our plug-in model in Sections \ref{sec:gadh} and \ref{sec:method-2}. We provide the information about the datasets, baseline models and hyperparameter tuning in Section \ref{sec:eval-setup} and Appendix \ref{app:details}. Furthermore, we provide our codebase as supplementary material with the submission.

\item {\bf Open access to data and code}
    \item[] Question: Does the paper provide open access to the data and code, with sufficient instructions to faithfully reproduce the main experimental results, as described in supplemental material?
    \item[] Answer: \answerYes{} % Replace by \answerYes{}, \answerNo{}, or \answerNA{}.
    \item[] Justification: We provide our codebase as supplementary material with the submission, which includes scripts to download the relevant datasets. All the datasets used in our submission are publicly available, as noted in Section \ref{sec:eval-setup}.

\item {\bf Experimental setting/details}
    \item[] Question: Does the paper specify all the training and test details (e.g., data splits, hyperparameters, how they were chosen, type of optimizer) necessary to understand the results?
    \item[] Answer: \answerYes{} % Replace by \answerYes{}, \answerNo{}, or \answerNA{}.
    \item[] Justification: We provide complete details about our plug-in model in Sections \ref{sec:gadh} and \ref{sec:method-2}. We provide the information about the datasets, baseline models and hyperparameter tuning in Section \ref{sec:eval-setup} and Appendix \ref{app:details}. Furthermore, we provide our codebase as supplementary material with the submission.

\item {\bf Experiment statistical significance}
    \item[] Question: Does the paper report error bars suitably and correctly defined or other appropriate information about the statistical significance of the experiments?
    \item[] Answer: \answerYes{} % Replace by \answerYes{}, \answerNo{}, or \answerNA{}.
    \item[] Justification: Our results in Section \ref{sec:results} includes the standard deviations and appropriate information.

\item {\bf Experiments compute resources}
    \item[] Question: For each experiment, does the paper provide sufficient information on the computer resources (type of compute workers, memory, time of execution) needed to reproduce the experiments?
    \item[] Answer: \answerYes{} % Replace by \answerYes{}, \answerNo{}, or \answerNA{}.
    \item[] Justification: We provide information about our computer resources in Appendix \ref{app:details}.
    
\item {\bf Code of ethics}
    \item[] Question: Does the research conducted in the paper conform, in every respect, with the NeurIPS Code of Ethics \url{https://neurips.cc/public/EthicsGuidelines}?
    \item[] Answer: \answerYes{} % Replace by \answerYes{}, \answerNo{}, or \answerNA{}.
    \item[] Justification: Our research conforms with NeurIPS Code of Ethics.

\item {\bf Broader impacts}
    \item[] Question: Does the paper discuss both potential positive societal impacts and negative societal impacts of the work performed?
    \item[] Answer: \answerNA{} % Replace by \answerYes{}, \answerNo{}, or \answerNA{}.
    \item[] Justification: This paper presents work whose goal is to advance the field of machine learning. There are many potential societal consequences of our work, none of which we feel must be specifically highlighted here.
    
\item {\bf Safeguards}
    \item[] Question: Does the paper describe safeguards that have been put in place for responsible release of data or models that have a high risk for misuse (e.g., pre-trained language models, image generators, or scraped datasets)?
    \item[] Answer: \answerNA{} % Replace by \answerYes{}, \answerNo{}, or \answerNA{}.
    \item[] Justification: The paper poses no such risks.

\item {\bf Licenses for existing assets}
    \item[] Question: Are the creators or original owners of assets (e.g., code, data, models), used in the paper, properly credited and are the license and terms of use explicitly mentioned and properly respected?
    \item[] Answer: \answerYes{} % Replace by \answerYes{}, \answerNo{}, or \answerNA{}.
    \item[] Justification: We have included relevant citations and licences for the code and data we use. The relevant licence is listed in Appendix \ref{app:details}.

\item {\bf New assets}
    \item[] Question: Are new assets introduced in the paper well documented and is the documentation provided alongside the assets?
    \item[] Answer: \answerYes{} % Replace by \answerYes{}, \answerNo{}, or \answerNA{}.
    \item[] Justification: We provide anonymised code as supplementary material.

\item {\bf Crowdsourcing and research with human subjects}
    \item[] Question: For crowdsourcing experiments and research with human subjects, does the paper include the full text of instructions given to participants and screenshots, if applicable, as well as details about compensation (if any)? 
    \item[] Answer: \answerNA{} % Replace by \answerYes{}, \answerNo{}, or \answerNA{}.
    \item[] Justification: The paper does not involve crowdsourcing nor research with human subjects.

\item {\bf Institutional review board (IRB) approvals or equivalent for research with human subjects}
    \item[] Question: Does the paper describe potential risks incurred by study participants, whether such risks were disclosed to the subjects, and whether Institutional Review Board (IRB) approvals (or an equivalent approval/review based on the requirements of your country or institution) were obtained?
    \item[] Answer: \answerNA{} % Replace by \answerYes{}, \answerNo{}, or \answerNA{}.
    \item[] Justification: The paper does not involve crowdsourcing nor research with human subjects.

\item {\bf Declaration of LLM usage}
    \item[] Question: Does the paper describe the usage of LLMs if it is an important, original, or non-standard component of the core methods in this research? Note that if the LLM is used only for writing, editing, or formatting purposes and does \emph{not} impact the core methodology, scientific rigor, or originality of the research, declaration is not required.
    \item[] Answer: \answerNA{} % Replace by \answerYes{}, \answerNo{}, or \answerNA{}.
    \item[] Justification: The core method development in this research does not involve LLMs.

\end{enumerate}

\end{document}